\documentclass[11pt]{article}

\usepackage{amsthm}
\usepackage{amssymb}
\usepackage{amsmath}
\usepackage{graphicx}
\usepackage{algorithm}
\usepackage{algorithmic}
 \usepackage{tcolorbox}
 \usepackage{geometry}
 \usepackage[titletoc,title]{appendix}

\newtheorem{theorem}{Theorem}
\newtheorem{lemma}{Lemma}
\newtheorem{definition}{Definition}

\newtheorem{remark}{Remark}

\usepackage[hidelinks]{hyperref}
\usepackage[numbers,sort&compress]{natbib}

\newcommand{\longversion}[1]{}

\renewcommand{\longversion}[1]{}

\begin{document}

\title{Risk-Averse Stochastic Convex Bandit}
\author{Adrian Rivera Cardoso \thanks{School of Industrial and Systems Engineering, Georgia Institute of Technology. {\tt adrian.riv@gatech.edu}. Supported in part by a TRIAD-NSF grant (award 1740776).} \and  Huan Xu \thanks{School of Industrial and Systems Engineering, Georgia Institute of Technology. {\tt huan.xu@isye.gatech.edu}.}}
 \maketitle

% !TEX root = aistats_main.tex

\begin{abstract}
Motivated by applications in clinical trials and finance, we study the problem of online convex optimization (with bandit feedback) where the decision maker is risk-averse. We provide two algorithms to solve this problem. The first one is a descent-type algorithm which is easy to implement. The second algorithm, which combines the ellipsoid method and a center point device, achieves (almost) optimal regret bounds with respect to the number of rounds. To the best of our knowledge this is the first attempt to address risk-aversion in the online convex bandit problem.  
\end{abstract}

\section{Introduction}
In this paper we study the problem of Online Risk-Averse Stochastic Optimization which generalizes Online Convex Optimization (OCO) when the loss functions are sampled i.i.d from an unknown distribution. During the last decade OCO has received a lot of attention due to its many applications and tight relations with problems such as Universal Portfolios \cite{cover1991, kalai2002, helmbold1998}, Online Shortest Path \cite{takimoto2003path}, Online Submodular Minimization \cite{hazan2012submodular}, Convex Optimization \cite{hazan2010optimal, ben2015oracle}, Game Theory \cite{cesa2006prediction} and many others. Along with OCO came Online Bandit Optimization (OBO) a similar but more challenging line of research, perhaps more realistic in some applications, where the feedback is limited to observing only the function {\em values} of the chosen actions (bandit feedback) instead of the whole functions \cite{flaxman2005online}. The standard goal of OCO and OBO is to develop algorithms such that the standard average regret
\begin{align*}
\frac{1}{T}\sum_{t=1}^T f_t(x_t) -\frac{1}{T} \min_{x\in X} \sum_{t=1}^T f_t(x)
\end{align*}
vanishes as quickly as possible. In other words, we want our average loss to be as close as possible to the best loss if we had known all the functions in advance and committed to one action. Here the sequence of convex functions  $\{f_t\}_{t=1}^T$ may be chosen by an adversary and the regret minimizing algorithm chooses action  $x_{t+1}$, in some bounded convex set $X$ by using only the information available at time $t$. This means that in the OCO setting the algorithm may use $\{x_1,...,x_t\}$ and $\{f_1(\cdot),...f_t(\cdot)\}$, and in the OBO setting it may only use $\{x_1,...,x_t\}$ and $\{f_1(x_1),...f_t(x_t)\}$. Due to recent breakthroughs \cite{bubeck2016kernel, hazan2016optimal} we now have efficient algorithms (that meet lower bounds in terms of the number of rounds $\Omega(\frac{1}{\sqrt{T}})$ up to logarithmic factors) for both problems, OBO and OCO. While the set up of OCO and OBO is very powerful because it allows for the loss functions to be chosen adversarially, in some applications such as medicine and finance this may not be enough.

Let us consider an example in clinical trials. Suppose there are $T$ patients with some rare disease and we have at our disposal a new drug that has the potential to cure the disease if we prescribe the right dose. Since we do not know what the right dose is, we must learn it as we treat each patient. In other words, we will choose a dose, observe the reaction of a patient and chose a new dose for the next patient. The previous problem can of course be be abstracted as an OBO problem, where each function $f_t(\cdot)$ encodes how patient $t$ will react to the dose we prescribe $x_t$. Here, the assumption that $f_t$ is chosen adversarially may not be very realistic and perhaps it makes more sense to assume that $f_t$ is drawn randomly from some family of functions. An algorithm that guarantees that the standard average regret vanishes can be seen as an algorithm that is choosing the optimal dose for the average patient, something that is non-trivial to do. Unfortunately, such guarantee completely ignores what may happen to patients that do not look like the average patient. It could be that the optimal dose for the average patient has really negative effects on $5\%$ of the patients. In this case, a dose that is slightly less effective on the average patient but does not harm the unlucky $5\%$ may be more desirable. Thus, the goal of this paper is to provide algorithms for OCO and OBO that {\em explicitly incorporate risk}. By ``risk" we mean the possibility of really negative outcomes, as it is used in the Economics and Operations Research communities.

Another area where an explicit consideration of risk must be taken into account is finance. For example, in \cite{even2006risk} the authors show that in the online portfolio problem, risk neutral guarantees such as performing as well as the best constant rebalanced portfolio (i.e. minimizing standard average regret) may not perform well in practice. They show through experiments on the S$\&$P500 that the simple strategy that maintains uniform weights on all the stocks outperforms that which seeks to perform as well as the best stock (regardless of its theoretical guarantees). To explicitly incorporate risk into the setting of OCO and OBO we will use a coherent risk measure called Conditional Value at Risk ($CVaR$) \cite{rockafellar_cvar}, sometimes also called Expected Shortfall, which is widely used in the financial industry. 
After the financial crisis of 2008, the Basel Committee on Banking Supervision created the Third Basel Accord (Basel III), a set of regulatory measures to strengthen the regulation, supervision and risk management of the banking sector \cite{basel_III}. In this accord one of the main points was to migrate from quantitative risk measures such as Value at Risk to Conditional Value at Risk since it better captures tail risk.

It should be clear from the previous examples that generally speaking, human decision makers are risk-averse. They prefer consistent sequences of rewards instead of highly variable sequences with slightly better rewards. Because of the previous, we want to develop algorithms that explicitly incorporate risk which have strong theoretical guarantees. 

Our main contributions are the following. First, we develop and analyze two algorithms for the online stochastic convex bandit problem that explicitly incorporate the risk aversion of the decision maker (as measured by the $CVaR$). On our way we develop a finite-time concentration result for the $CVaR$. Second, we extend our results to the case where the decision maker uses more general risk measures to measure risk by using the Kusuoka representation theorem. 

\section{Related Work}
Risk aversion has received very little attention in the online learning setting. The few existing work all focuses on the case where {\em the number of actions is finite}. For the  stochastic multi-armed bandit problem, \cite{sani2012risk} provide algorithms that ensure the mean-variance of the sequence of rewards generated by the algorithm is not too far from the mean-variance of the rewards generated by the best arm. In \cite{vakili2013deterministic} the same problem is studied and the authors provide tighter upper and lower bounds. In \cite{maillard2013robust} the author considers a different risk measure, the cumulant  generative function, and provide similar guarantees for a slightly modified definition of regret. In \cite{galichet2013exploration} the authors consider the $CVaR$ as measure of risk aversion and provide algorithms that achieve sublinear regret. The notion of regret they use is different from the one we will use as they do not look at the risk of the sequence of rewards obtained by the algorithms, but instead they seek to perform as well as the arm that minimizes $CVaR$ (i.e.,  ``pseudo regret'' as we called). The pseudo regret bound they prove, although optimal with respect to $T$ scales linearly in the number of arms. By using a discretization approach in our setting together with their algorithm would yield an algorithm with pseudo regret that depends exponentially in the dimension of the problem with exponential running time. The previous is of course undesirable, therefore different tools must be used. In \cite{yu2013sample} the authors study the related problem of best arm identification where the goal is to identify the arm with the best risk measure. They consider Value at Risk, $CVaR$, and Mean-Variance as risk measures. 
In \cite{even2006risk} the authors consider risk aversion in the experts problem. This setting is similar to the multi-armed bandit problem with the difference that the rewards are assigned adversarially, and at each time step all the rewards are visible to the player. In particular they seek to build algorithms such that the mean variance (or Sharpe ratio) of the sequence of rewards generated by the algorithm are as close as possible to that of the best expert. They show negative results for this problem however they provide algorithms that perform well for ``localized" versions of the risk measures they consider.

To the best of our knowledge, all existing work that explicitly incorporates risk aversion under the assumptions of stochastic rewards and bandit feedback is restricted to the multi-armed bandit model. This paper is the first to consider an infinite number of arms and incorporate risk aversion under bandit feedback. In \cite{even2006risk}, where risk aversion in the experts problem is studied, one can think of instead of choosing an expert at every round one chooses a probability distribution over the experts. While the set of probability distributions over the experts is a convex set, this is a very specialized case (linear functional and simplex feasible set). Moreover, the authors assume full information feedback and adversarial rewards, which  are very different from our setup.

\section{Preliminaries}
This section is devoted to preliminaries. In particular we review relevant concepts and  technical results essential to develop the proposed algorithms.

\subsection{Notation}
Let $||\cdot||$ be the $l_2$ norm unless otherwise stated. By default all vectors are column vectors, a vector with entries $x_1,...,x_n$ is written as $x = [x_1;...;x_n] = [x_1,...,x_n]^\top$ where $\top$ denotes the transpose. For a random variable $X$, $X\sim P$ means that $X$ is distributed according to distribution $P$. We let $\nabla g(x)$ be any element in the subdifferential of $g$ at $x$. Whenever we write $\nabla f(x,\xi)$ we mean $\nabla_x f(x,\xi)$. Throughout the paper we will use $O$ notation to hide constant factors. We use $\tilde{O}$ notation to hide constant factors and poly-logarithmic factors of the number of rounds $T$, the inverse risk level $\frac{1}{\alpha}$ and the dimension of the problem $d$.

\subsection{One-Point Gradient Estimation}
Consider function $f:\mathbb{R}^d \rightarrow \mathbb{R}$ which is $G$-Lipschitz continuous. Define its smoothened version
 \begin{align*}
 \hat{f}^\delta(x) := \mathbb{E}_{v \sim \mathbb{B}}[f(x+\delta v)]
 \end{align*}
 where $\mathbb{B}$ is the uniform distribution over the unit ball of appropriate dimension. From now on we omit superscript $\delta$ and write $\hat{f}(x)$. Define random quantity 
 \begin{equation}\label{one_point_estimate}
 g = \frac{d}{\delta}f(x+\delta u)u 
 \end{equation}
 with $u \sim \mathbb{S}$ where $\mathbb{S}$ is the uniform distribution over the unit sphere. We have the following
  
\begin{lemma}\label{smooth_f}
\cite{hazan2016introduction}[Ch.2] $\hat{f}$ satisfies the following:
 \begin{enumerate}
 \item If $f$ is $\alpha$-strongly convex then so is $\hat{f}$
 \item $|f(x)- \hat{f}(x)| \leq \delta G$
 \item $\mathbb{E}[g] = \nabla \hat{f}(x)$
 \end{enumerate}
 \end{lemma}
 That is, the smoothened version of $f$ is convex as well, it is not too far from $f$, and by sampling from the unit sphere we can obtain an unbiased estimate of its gradient. 

\subsection{Conditional Value at Risk}
In \cite{rockafellar_cvar} the authors define the $\alpha$-Value at Risk of random variable $X$ as
\begin{align*}
VaR_\alpha[X] := \inf\{t : P(X \leq  t) \geq1-\alpha\}.
\end{align*}
Using the above definition they define Conditional Value at Risk ($CVaR$, sometimes also called Expected Shortfall) as
\begin{align}
C_{\alpha}[X] := CVaR_{\alpha}[X] := \frac{1}{\alpha} \int_{1-\alpha}^{\alpha} VaR_{1-\tau}[X]d\tau.
\end{align}
Moreover, when the random variable has c.d.f. $H(x)$ continuous at $x=VaR_\alpha[X]$ it holds that
\begin{align}
C_{\alpha}[X] = \mathbb{E}[X | X\geq VaR_{\alpha}[X]].
\end{align}
Below we state some well known results that will be used later. The proofs for the next two lemmas can be found in \cite{shapiro2009lectures}.

 \begin{lemma} 
 \begin{align}
 C_\alpha[X] = \min_{z\in \mathbb{R}} z +\frac{1}{\alpha} \mathbb{E}[X-z]_+,
 \end{align}
where $[a]_+ := \max\{a,0\}$. In fact, if $0\leq X \leq 1$ with probability 1, the condition $z \in \mathbb{R}$ can be replaced with $z \in [0,1]$.
\end{lemma}
 
 \begin{lemma}
 \label{cvarf}
Let $\xi$ be a random variable supported in $\Xi$ with distribution $P$, let $X \subset \mathbb{R^d}$ be a convex and compact and let $f:X \times \Xi \rightarrow \mathbb{R}$ be convex in $x$ for every $\xi$. Define $F = f(x,\xi)$. Then
 \begin{align*}\label{cvar_as_min}
 C_{\alpha}[F](x):=CVaR_\alpha[F](x) = \min_{z} z + \frac{1}{\alpha}\mathbb{E}_{\xi}[f(x,\xi)-z]_+
 \end{align*}
 and $C_\alpha[F](x)$ is a convex function of $x$. In fact, if $f(\cdot,\xi)$ is $\beta$-strongly convex for every $\xi \in \Xi$, then so is $C_{\alpha}[F](x).$ 
 \end{lemma}

\section{Problem Setup} \label{setup}
In this section we formally define the setup of our problem. Let $\xi$ be a random variable supported in $\Xi$ with unknown distribution $P$. Let $X \subset \mathbb{R}^d$ be a convex and compact set with diameter $D_X$ that contains the origin. Let $f:X \times \Xi \rightarrow \mathbb{R}$ be a convex function in the first argument for every $\xi \in \Xi$. Let $f$ satisfy $||\nabla f(x,\xi)||\leq G$ for every $x\in X$ and every $\xi \in \Xi$. We define random function $F(x)= f(x,\xi)$ in the sense that for every $x\in X$, $F(x)$ is a random variable. We also assume that for every $x\in X$, $0\leq F(x) \leq 1$ with probability 1.
 
A risk-averse player will make decisions in a {\em stochastic environment} for $T$ time steps. In every time step $t=1,...,T$ the player chooses action $\tilde{x}_t \in X$, and nature obtains sample $\xi_t$ from $P$. Then, the player incurs and observes only the loss incurred by its action $f(\tilde{x}_t,\xi_t)$. If the player were risk neutral then a reasonable goal would be to design an algorithm that obtains (in expectation) vanishing standard average regret, that is
\begin{equation*}
\mathbb{E}[\frac{1}{T} \sum_{t=1}^T f(\tilde{x}_t,\xi_t) - \frac{1}{T} \min_{x\in X} \sum_{t=1}^T f(x,\xi_t)] = o(1).
\end{equation*}
Where the expectation is taken with respect to the random draw of functions and the internal randomization of the algorithm. Such is the standard goal of OCO and OBO, and as mentioned in the introduction, there already exist polynomial time algorithms that achieve the optimal lower bound of $\Omega(1/\sqrt{T})$ (up to logarithmic factors) even when the functions $f$ are chosen by an adversary instead of from some distribution.

In our setting, since the player is risk averse, the notion of  average regret is not appropriate. In this section we assume that the player uses the Conditional Value at Risk $C_\alpha[\cdot] = CVaR_\alpha[\cdot]$ for some $\alpha \in(0,1]$ to measure risk (when $\alpha = 1$, $C_\alpha[\cdot] = \mathbb{E}[\cdot]$ i.e. the player becomes risk neutral). With this in mind, the following two quantities become interesting, namely pseudo-$CVaR$-regret defined as
\begin{equation}\label{eq:ps_cvar_reg}
\bar{\mathcal{R}}_T := \frac{1}{T}\sum_{t=1}^T C_\alpha[F](\tilde{x}_t) - \frac{1}{T} \min_{x\in X} \sum_{t=1}^T C_\alpha[F](x)
\end{equation}
and $CVaR$-regret defined as 
 \begin{equation*}
 \mathcal{R}_T := C_\alpha[\{f_t(\tilde{x}_t)\}_{t=1}^T] - \min_{x\in X}C_\alpha[\{f_t(x)\}_{t=1}^T],
 \end{equation*}
where we make more explicit what we mean by $C_\alpha[\{f_t(x_t)\}_{t=1}^T]$ in the next paragraph. 
In this setup, a risk averse player may be concerned with two types of risk, the risk of the individual losses it incurs and the overall risk of playing the game. The player that is concerned about the risk of the individual losses, should be pleased with an algorithm that obtains vanishing $\bar{\mathcal{R}}_T$, this would ensure that the average risk of the losses it incurs is not too far from that of the best point in the set. 

On the other hand, the player that is concerned about the overall risk of playing the game may desire a different guarantee. Notice that the sequence of losses that the player incurs $\{f_t(\tilde{x}_t)\}_{t=1}^T$ defines an empirical distribution where every realization $f_t(\tilde{x}_t)$ occurs with probability $\frac{1}{T}$ and as such we can compute its risk $C_\alpha[\{f_t(\tilde{x}_t)\}_{t=1}^T]$. It is then natural for the player to desire a sequence of losses that has risk as close as possible to the minimum risk sequence of losses (where the sequence is generated by playing only one action). The quantity $\mathcal{R}_T$ makes the previous statement precise.

 A reader familiar with the OBO literature may notice that (\ref{eq:ps_cvar_reg}) already looks like a quantity for which running Online Gradient Descent without a Gradient may yield vanishing regret. Unfortunately, at every step all we observe is $f_t(\tilde{x}_t)$ and not $C_\alpha[F](\tilde{x}_t)$. To obtain a reasonable (not too noisy) evaluation of $C_\alpha[F](\cdot)$ the same $x$ must be played for several rounds. It is possible to design algorithms that follow this idea, however, since we were able to develop better algorithms for the same problem we do not further discuss the details of this somewhat naive approach. 
 
 \section{A Finite-Time Concentration Result for the $CVaR$}\label{sect_cvar}
Before we present the algorithms we must derive a finite-time concentration result for the $CVaR$. This result will be heavily used to prove sublinear regret bounds for both algorithms. In \cite{shapiro2009lectures} the authors present an asymptotic result. Unfortunately, since our goal is to achieve finite-time bounds we could not use it and had to prove our own result. To the best of our knowledge this is the first finite time concentration result for the $CVaR$. 

\begin{theorem}
\label{concentration_cvar}
Suppose $0\leq f(x,\xi)\leq1$  for every $x\in X$ and every $\xi \in \Xi$ . For any $x \in X$, let the N-sample estimate of $CVaR_\alpha[F](x)$ be  $\widehat{CVaR_\alpha[F](x)} := \min_{z \in Z } z + \frac{1}{\alpha N}\sum_{n=1}^N [f(x,\xi_n) - z]_+$. Where $Z:=[0,1]$. It holds that with probability at least $1-\delta$,
\begin{align*}
|CVaR_\alpha[F](x) - \widehat{CVaR_\alpha[F](x)} | \leq O( \sqrt{\frac{ \ln(N / \delta)}{\alpha^2 N}}  ). 
\end{align*} 
\end{theorem}
While the previous result holds with high probability it is also possible to derive from it a result that holds in expectation. 

To prove such a  result we had to use a finite time concentration result for Lipschitz functions from \cite{shalevstochastic} applied to the sequence of functions $\{z + \frac{1}{\alpha}[f(x,\xi_t)-z]_+ \}_{t=1}^T$. After this, some extra work had to be done transform this guarantee into one that holds for the $CVAR$. A formal proof of the theorem can be found in the appendix.
  
\section{Algorithm 1} \label{algorithm_1}
In this section we provide an algorithm that obtains vanishing regret while playing an action only once. The key to the algorithm is to look at functions $\mathcal{L}_t(x,z):= z+\frac{1}{\alpha}[f(x,\xi_t)-z]_+$ which by Lemma \ref{cvarf} are closely related to $C_\alpha[F](x)$. Although with one sample we can not evaluate (accurately enough) $C_\alpha[F](\cdot)$, we can evaluate $\mathcal{L}_t$. This observation is important because it will allow us to build one-point gradient estimators of the smoothened function $\hat{\mathcal{L}_t}$ as it is done in \cite{flaxman2005online}. These one-point gradient estimators will allow us to perform a descent step. This idea allows us to obtain sublinear pseudo-regret. The rest of the analysis consists of using the bound on the pseudo-regret to bound the regret.

 \begin{algorithm}[tbh]
   \caption{}
   \label{alg:algo_1}
\begin{algorithmic}
   \STATE {\bfseries Input:} $X \subset \mathbb{R}^d$, $x_1 \in X$, $z_1 \in Z:=[0,1]$ step size $\eta$, $\delta$
   \FOR{$t=1,...,T$}
   \STATE Sample $u \sim \mathbb{S}^{d+1}$
   \STATE Let $u^1 = [u_1;...;u_d]$ and $u^2 = u_{d+1}$
   \STATE Play $\tilde{x}_t:=x_t + \delta u^1$, incur and observe loss $f_t(\tilde{x}_t)$ 
   \STATE Let $ \tilde{z}_t = z_t + \delta u^2 $
   \STATE Let $g^1_t := \frac{(d+1)}{\delta} (\tilde{z}_t+ \alpha^{-1} [f_t(\tilde{x}_t) - \tilde{z}_t ]_+) u^1 $
   \STATE Let $g^2_t := \frac{(d+1)}{\delta} (\tilde{z}_t+ \alpha^{-1} [f_t(\tilde{x}_t) - \tilde{z}_t ]_+) u^2 $\\
    \STATE Update $x_{t+1} \leftarrow \Pi_{X_\delta} (x_t - \eta g^1_t)$\\
    \STATE Update $z_{t+1} \leftarrow \Pi_{Z_\delta} (z_t - \eta g^2_t)$\\
   \ENDFOR
\end{algorithmic}
\end{algorithm}

Here $\mathbb{S}^d$ denotes the uniform distribution over the $d$-dimensional unit sphere, $X_\delta := \{x : \frac{1}{1-\delta}x \in X \}$ and $\Pi_{X}[\cdot]$ denotes the  $||\cdot||_2$ projection onto convex set X.

We have the following two main results. 
\begin{theorem}\label{thm:pseudo_algo1} Using $\eta = \frac{\alpha D_{\mathcal{L}} }{(d+1) T^{3/4}} $ and $\delta = \frac{1}{T^{1/4}}$ Algorithm 1 guarantees:
\begin{align*}
 \mathbb{E}[\bar{\mathcal{R}}_T]  \leq O(\frac{d}{\alpha T^{1/4}}).
 \end{align*}
Where the expectation is taken over the random draw of functions and the internal randomization of the algorithm. $D_\mathcal{L}$ is specified in the appendix. 
\end{theorem}

\begin{theorem}\label{thm:regret_algo1}
Let $f(x,\xi)$ be strongly convex with parameter $\beta>0$. Algorithm 1 guarantees
\begin{align*}
\mathbb{E} [�\mathcal{R}_T ]\leq \tilde{O} (\frac{d^{1/2}}{\alpha^{3/2} \beta^{1/2} T^{1/8}} ).
\end{align*}  
Where the expectation is taken over the random draw of functions and the internal randomization of the algorithm.
\end{theorem}
The proofs of these theorems can be found in the appendix. 

\section{Algorithm 2}\label{algorithm_2}
Algorithm 1, while it is intuitive and easy to implement, does not achieve the optimal pseudo-regret bound of $\frac{1}{\sqrt{T}}$. In this section, we adapt an algorithm from \cite{agarwal2011stochastic} that achieves the optimal regret bound (up to logarithmic factors), unfortunately its dependency on $d$ is less than ideal. We consider the cases $d=1$ and $d>1$ separately. 

\subsection{The 1-Dimensional Case}
For simplicity, in this section we assume that $X=[0,1]$ and that $f(\cdot,\xi)$ is $1$-Lipschitz continuous for every $\xi \in \Xi$. This implies that $C_\alpha[F](\cdot)$ is also $1$-Lipschitz continuous (see Lemma \ref{C_alpha_Lipschitz} in the appendix). We let $LB_{\gamma_i}(x)$ and $UB_{\gamma_i}(x)$ denote the $C_{\alpha}[F](\cdot)$ lower and upper bounds of the confidence intervals (CI's) of width $\gamma_i$ at point $x$. That is, sample point $x$ $\frac{\ln(T/(\alpha \gamma))}{\gamma_i^2 \alpha^2}$ times, compute the empirical $CVaR_\alpha$, $\hat{C}_\alpha[F](x)$ and let $UB_{\gamma_i}(x) := \hat{C}_\alpha[F](x) + \gamma_i $ and  $LB_{\gamma_i}(x) := \hat{C}_\alpha[F](x) - \gamma_i $. 

\begin{algorithm*}[tbh]
   \caption{$(d=1)$}
   \label{alg:algo_2}
\begin{algorithmic}
   \STATE {\bfseries Input:} Input: $X\in [0,1]$, total number of time-steps $T$
   \STATE Let $l_1:=0, r_1:=1$   
   \FOR{epoch $\tau=1,2,...$}
   \STATE Let $w_\tau := r_{\tau}- l_{\tau}$
   \STATE Let $x_{l}:= l_{\tau} + w_{\tau}/4, x_c = l_{\tau} + w_{\tau}/2, x_r := l_{\tau}+3w_{\tau}/4 $
   \FOR{round $i=1,2,...:$}
   \STATE Let $\gamma_i = 2^{-i}$
   \STATE For each $x \in \{x_{l}, x_{c}, x_{r}\}$ play $x$ $\frac{\ln(T/(\alpha \gamma))}{\gamma_i^2 \alpha^2}$ times and build CI's: ${[\hat{C}_\alpha[F](x_k)]-\gamma_i, \hat{C}_\alpha[F](x_k)+\gamma_i]}$ for $k\in\{l,c,r\}$
   \IF{$\max\{LB_{\gamma_i}(x_l), LB_{\gamma_i}(x_r)\} \geq \min\{UB_{\gamma_i}(x_l), UB_{\gamma_i}(x_r)\}+\gamma_i$ (Case 1)}
   	\IF{$LB_{\gamma_i}(x_l)\geq LB_{\gamma_i}(x_r)$}
   		\STATE set $l_{\tau +1}:=x_l$ and $r_{\tau+1}:= r_{\tau}$ 
   	\ELSE 
   		\STATE set $l_{\tau +1}:=l_{\tau}$ and $r_{\tau+1}:= x_r$
   	\ENDIF
   	\STATE Continue to epoch $\tau +1$
   \ELSIF{$\max\{LB_{\gamma_i}(x_l), LB_{\gamma_i}(x_r)\} \geq UB_{\gamma_i}(x_c) + \gamma_i $  (Case 2)}
    	\IF{$LB_{\gamma_i}(x_l)\geq LB_{\gamma_i}(x_r)$}
		\STATE set $l_{\tau +1}:=x_l$ and $r_{\tau+1}:= r_{\tau}$
	\ELSE 
		\STATE set $l_{\tau +1}:=l_{\tau}$ and $r_{\tau+1}:= x_r$
	\ENDIF
	\STATE Continue to epoch $\tau +1$
   \ENDIF{ (Case 3)} 
   \ENDFOR
    \ENDFOR
\end{algorithmic}
\end{algorithm*}

The algorithm proceeds in epochs and rounds. In epoch $\tau$ the algorithm works with region $[l_{\tau},r_{\tau}]$. In this region we will be playing three points $x_l, x_c, x_r$ ($x_c$ is the center point) for several rounds $i=1,2,...$ . In each round $i$ the algorithm will play $\frac{\ln(T/(\alpha \gamma))}{\alpha^2 \gamma_i^2}$ times the aforementioned points and build CI's for $C_{\alpha}[F]$. Roughly speaking, the reason why the algorithm works is because in every round we are 1) either playing points such that we are not suffering too much pseudo-regret or 2) we are quickly identifying a subregion of the working region which only contains ``bad points" and discarding it.  Every time 2) occurs we are shrinking the working region by a constant factor, this will guarantee that after not too many rounds we are only working with a small feasible region.

For convenience  we denote $h(x):= C_{\alpha}[F](x)$ and $x^*:=\text{argmin}_{x\in X} h(x)$. Notice that the minimizer need not be unique in which case we choose one arbitrarily. At the end of a round one of the following occurs:

\textbf{Case 1.} The CI's around $h(x_l)$ and $h(x_r)$ are sufficiently separated. If this is the case, then by convexity we can discard one fourth of the working feasible region: either the one to the left of $x_l$ or the one to the right of $x_r$ .

 \textbf{Case 2.} If Case 1 does not occur, the algorithm checks if the CI around $h(x_c)$ is sufficiently below at least one of the CI's around $h(x_l)$ or $h(x_r)$. If this is the case then we can discard one fourth of the working feasible region.
 
\textbf{Case 3.} If neither Case 1 or Case 2 occurs then we can be sure that the function is flat in the working feasible region (as measured by $\gamma$) and thus we are not incurring a very high pseudo-regret.

The main results of this section are the following. 
\begin{theorem}
\label{thm:first_theorem_algo2_1d}
With probability at least $1-\frac{1}{T}$, Algorithm 2 (1-D) guarantees
\begin{align*}
\bar{\mathcal{R}}_T \leq O(\frac{ \ln(T)}{\sqrt{T} \alpha} \ln(\frac{\alpha T}{\ln(T)})). 
\end{align*}
\end{theorem}

\begin{theorem}
\label{thm:second_theorem_algo2_1d}
Let $f(\cdot,\xi)$ be strongly convex with parameter $\beta >0$ for all $\xi \in \Xi$. With probability at least $1 - \frac{3}{T}$, Algorithm 2 (1-D) guarantees
\begin{align*}
\mathcal{R}_T \leq \tilde{O}(\frac{1}{\alpha^{3/2} \beta^{1/2} T^{1/4}}).
\end{align*}
\end{theorem}

We follow \cite{agarwal2011stochastic} for the analysis of the algorithm. The main difference in the analysis is that we must build estimates of the $CVaR$ of the random loss at every point instead of building them for the expected loss. Because of this, we have to use our concentration result from Section \ref{sect_cvar}. This directly affects how many times we must choose an action.The detailed analysis of the algorithm and the proofs of the theorems in this section can be found in the appendix. 
  
\subsection{The $d$-Dimensional Case}

%The algorithm for the $d>1$ case combines 1) an ellipsoid-like method from  \cite{nemirovskii1983problem} that works with noisy zeroth-order information (i.e. noisy function evaluations), 2) a center point device from \cite{agarwal2011stochastic} that ensures that we are not incurring to much regret when using the ellipsoid-like algorithm, and 3) our finite time concentration result of $CVaR$. For this algorithm we could obtain similar guarantees as in Theorems \ref{thm:first_theorem_algo2_1d} and \ref{thm:second_theorem_algo2_1d}. A more detailed description of the algorithm, its theoretical guarantees, and its analysis can be found in the appendix. 
% \section{Algorithm 2 ($d$-D)}
%Next, we give some intuition on how the algorithm works.
Let us first consider the problem of minimizing a convex function over a bounded set with a first-order oracle (i.e. a gradient and function value oracle). For simplicity let us assume that the convex set is a ball. An ellipsoid-type method would work really well in this setup because of the following. By querying the first order oracle at any point (due to convexity) we could identify a subregion of the current feasible region where the function value is worse than the function value at the point we made the query. If we could somehow discard that bad portion of the feasible set, and the size of this bad region is big enough, by iterating the procedure (assuming this can be done) we should end up with a set that only has points close to optimal.

Let us now consider a similar but harder problem of minimizing a convex function over a a bounded set (say a ball) with a zeroth-order oracle (i.e. a function value oracle). In this setup, with one query, we can no longer identify a subregion of the current feasible region where the function values are worse than the function value at the point we made the query. A first approach to tackle this problem is the following. Build a small regular simplex centered at the origin of the ball and query the function at its vertices. Assume the maximal function value occurs at vertex $y'$, then by convexity of the function one can conclude that the cone generated by reflecting the simplex around $y'$ is a region where the function values are bad. Since we have identified a bad region of the feasible set we would like to discard it and keep iterating our method, unfortunately what remains of the ball when we discard the cone is a non-convex set we can not keep iterating the method. To try to fix the previous one could try to find the minimum volume enclosing ellipsoid of the non-convex set and keep iterating. Unfortunately this does not work since the minimum volume enclosing ellipsoid will not have sufficiently small volume \cite{nemirovskii1983problem}. The reason this occurs is that the angle of the cone generated by reflecting the simplex around $y'$ is not wide enough. In \cite{nemirovskii1983problem} the authors fix the previous by constructing a pyramid (with wide enough angle) with $y'$ as its apex and sample the vertices of the pyramid. If we are lucky enough and $y'$ has the maximal function value among all the vertices of the pyramid, we can then discard the cone generated by reflecting the pyramid around $y'$ and enclose that region in the minimum volume ellipsoid. However, if we were not lucky enough and $y'$ did not have the maximal function value then, Nemirovski and Yudin \cite{nemirovskii1983problem}, show that by repeatedly building a new pyramid with apex at the point with maximal function value we will identify a bad region after building not too many pyramids. It is not to hard to see that the previous approach may work even if we have a noisy-zeroth-order oracle, as long as the noise is not too large. The previous approach describes an optimization procedure but by itself it does not guarantee low regret. However, by incorporating center points as done in \cite{agarwal2011stochastic}, sublinear regret can be achieved. Due to a lack of space the algorithm and its analysis can be found in the appendix.
The main results from this section are the following. 

\begin{theorem}\label{thm:pseudo_regret_algo2_full_dim}
Algorithm 2 ($d$-D) run with parameters $c_1\geq 64, c_2\leq 1/32$ and 
\begin{align*}
\Delta_{\tau}(\gamma) = \big(\frac{6c_1d^4}{c_2^2} +3 \big) \gamma, \quad \bar{\Delta}_{\tau}(\gamma) = \big( \frac{6c_1d^4}{c_2^2}+5\big)\gamma,
\end{align*}
guarantees that with probability at least $1-\frac{1}{T}$
\begin{align*}
%\bar{\mathcal{R}}_T \leq \frac{\kappa d^3  \ln(T/\alpha) \ln(T)}{\sqrt{T} \alpha^2} \big( \frac{2d^2\ln(d)}{c_2^2}+1\big) \big( \frac{4d^7c_1}{c_2^3} + \frac{d(d+2)}{c_2} \big) \big( \frac{12c_1d^4}{c_2^2} + 11\big).\\
\bar{\mathcal{R}}_T \leq \tilde{O}(\frac{d^{16}}{\alpha^2 \sqrt{T}}).
\end{align*}
\end{theorem}

\begin{theorem}\label{thm:regret_algo2_full_dim}
Let $f(\cdot ,\xi)$ be strongly convex with parameter $\beta>0$ for any $\xi \in \Xi$, Algorithm 2 ($d$-D) run with the same parameters as in Theorem \ref{thm:pseudo_regret_algo2_full_dim} guarantees that with probability at least $1-\frac{3}{T}$
\begin{align*}
\mathcal{R}_T \leq \tilde{O} (\frac{d^8 }{\alpha^3 \beta^{1/2} T^{1/4}}).
\end{align*}
\end{theorem}

%%%%%%%%%%%%%%MORE GENEARAL RISK  MEASURES%%%%%%%%%%%%%%%%%
\section{Extension to More General Risk Measures}

In Sections \ref{algorithm_1} and \ref{algorithm_2} we developed regret minimization algorithms suitable for decision makers who are risk averse, where the notion of risk was measured using the $CVaR_\alpha$. In this section we extend our results to more general risk measures.  
We slightly modify the setup from Section \ref{setup}. Now, we assume $\xi$ is a discrete random variable supported in $\Xi$ with $|\Xi| = N$. That is, there are $N$ scenarios. Moreover we assume that each scenario has the same probability of occurring. Let $X \subset \mathbb{R}^d$ be a convex and compact set. Let $f:X \times \Xi \rightarrow \mathbb{R}$ be a convex function in the first argument for every $\xi \in \Xi$. Let $f$ satisfy $||\nabla f(x,\xi) || \leq G$ for every $\xi \in \Xi $ and every $x \in X$. Additionally, we assume $0\leq f(x,\xi )\leq 1$ for every $x\in X$ and every $\xi \in \Xi$. We consider some law invariant, coherent and comonotone risk measure $\rho(\cdot)$ (see next subsection). Our goal now is to obtain vanishing pseudo-$\rho$-regret
\begin{align*}
\bar{\mathcal{R}}_T^\rho := \frac{1}{T} \sum_{t=1}^T \rho[F](x_t) - \frac{1}{T} \min_{x\in X} \sum_{t=1}^T \rho[F](x),
\end{align*} 
and $\rho$-regret
\begin{align*}
\mathcal{R}_T^\rho := \rho[\{f_t(x_t)\}_{t=1}^T] - \min_{x\in X} \rho[\{f_t(x_t)\}_{t=1}^T] .
\end{align*}
In this section we will show that by using the Kusuoka Representation Theorem along with the ideas we developed earlier we can obtain vanishing $\bar{\mathcal{R}}_T^\rho$ and $\mathcal{R}_T^\rho$.

\subsection{Kusuoka Representation of Risk Measures }
%In this section we consider some risk measures and present a classical result from Kusuoka \cite{kusuoka2001law}. We restrict ourselves to a finite probability space.
Before presenting the algorithms we present some necessary definitions and well known results.

\begin{definition}
A risk measure $\rho: \mathcal{X}(\Omega, 2^\Omega, P) \rightarrow \mathbb{R}$ is coherent if  for every $X_1, X_2 \in \mathcal{X}$ it is:
 \begin{itemize}
 \item Normalized, $\rho(0) = 0.$
 \item Monotone, $X_1 \leq X_2 \implies \rho(X_1) \leq \rho(X_2).$
 \item Superadditive, $\rho(X_1) + \rho(X_2) \leq \rho(X_1 + X_2).$
 \item Positive homogenous, $\rho(\lambda X_1)=\lambda \rho(X_1), \forall \lambda> 0.$
 \item Translation invariant, $\rho(X_1 + c) = \rho(X_1) + c.$
 \end{itemize} 
\end{definition}
 Moreover, we say $\rho $ is {\em law invariant} if $\rho(X_1)$ depends only on the distribution of $X_1$. Additionally, we say $\rho$ is {\em comonotone additive} if $\rho(X_1 + X_2) = \rho(X_1) + \rho(X_2)$. 
 
It is well known \cite{acerbi2002coherence} that $CVaR$ is a coherent risk measure. Indeed many risk measures can be expressed as functions of $CVaR$  \cite{pichler2012uniqueness}. We present a special case of the Kusuoka representation theorem that will be useful later. 

\begin{lemma}
\cite{noyan2015kusuoka} Consider  a finite probability space $(\Omega, 2^\Omega, P)$, with $\Omega = \{\omega_1,...,\omega_N\}$, and $P(\omega_n) = \frac{1}{N}$ for all $n=1,..., N$. Then, a mapping $\rho: \mathcal{X}(\Omega, 2^\Omega, P) \rightarrow \mathbb{R}$ is a law invariant coherent and comonotone additive risk measure if and only if it has a Kusuoka representation of the form 
 \begin{equation}\label{kusuoka_rho}
 \rho(X) =  \sum_{n=1}^N \mu_n CVaR_{\frac{n}{N}}(X), \quad \forall X \in \mathcal{X}
 \end{equation}
 where $\mu \in [0,1]^N$ and $||\mu||_1 = 1$. 
 \end{lemma} 
 \cite{pichler2012uniqueness} give examples on how the Kusuoka representation theorem can be used, in particular how to write the following risk measures as mixtures of $CVaR$'s. We refer the reader to their paper for the details. 
\begin{itemize}
\item $\rho(Z) :=  \inf_{t\in \mathbb{R}} \{t + c|| [Z-t]_+ ||_p\}, \quad \forall Z\in \mathcal{L}^p(\omega, \mathcal{F}, P)$ with $c>1$ and $1<p<\infty$. 
\item  $\rho(Z) := \mathbb{E}[Z] + \lambda || [Z - \mathbb{E}[Z] ]_+||$ for $p\geq 1$ and $0\leq \lambda \leq 1$. 
\end{itemize}

\subsection{Algorithms}

We define for every $t=1,...,T$, function $\mathcal{G}_t(x,z): X \times Z \rightarrow \mathbb{R}$, with $Z:= [0,1]^N$, as
\begin{align*}
\mathcal{G}_t(x,z):= \sum_{n=1}^N \mu_{n} ( z_{n} + \frac{1}{n/N} [f_t(x)-z_{n}]_+)
\end{align*}
for some $\mu \in [0,1]^N, \mu \geq 0, ||\mu||_1=1$. For convenience we write $\mathcal{L}_{n}^t(x,z) := z_{n} + \frac{1}{n/N} [f_t(x)-z_{n}]_+ $ for $n=1,...,N$. Notice that for any $x\in X$, after taking expectation with respect to $\xi$ and plugging the minimizer of every individual term $\mathcal{L}_n^t$ we end up with the Kusuoka representation of a law invariant, coherent and commonotone risk measure. Let $\mu$ be the vector corresponding to the Kusuoka representation of our risk measure of interest $\rho$ (see Equation \eqref{kusuoka_rho}). Algorithm 3, a generalization of Algorithm 1 that uses functions $\mathcal{G}_t$ instead of $\mathcal{L}_t$ can be found in the appendix. We have the following guarantees for Algorithm 3.

\begin{theorem}\label{thm:first_thm_Kusuoka}
Algorithm 3 with $\eta = O(\frac{1}{d N^{3/2} T^{3/4}})$ and $\delta = O(\frac{N^{1/2}}{T^{1/4}})$guarantees
\begin{align*}
\mathbb{E}[\bar{\mathcal{R}}^{\rho}_T] \leq O (\frac{d N^{3/2}}{T^{1/4}}),
\end{align*}
where the expectation is taken over the random draw of functions and the internal randomization of the algorithm.
\end{theorem}

\begin{theorem} \label{thm:second_thm_Kusuoka}
Let $f(\cdot,\xi)$ be strongly convex with parameter $\beta >0$ for all $\xi \in \Xi$. Algorithm 3, run with the same parameters as in Theorem \ref{thm:first_thm_Kusuoka}, guarantees 
\begin{align*}
\mathbb{E}[ \mathcal{R}^{\rho}_T] \leq O( \frac{d^{1/2} N^{7/4}}{\beta^{1/2} T^{1/8}}),
\end{align*}
where the expectation is taken over the random draw of functions and the internal randomization of the algorithm.
\end{theorem}

To obtain a better dependence on the number of rounds $T$, Algorithm 2 (in both cases, $d=1$ and $d>1$) can be modified to solve this more general problem. The only modification is that we will sample $\tilde{O}(\frac{N^2 \ln (\sqrt{N} T)}{\gamma})$ times a point to build a $\gamma$-CI for $\rho[F](x)$ for any $x\in X$. Let this modification of Algorithm 2 be Algorithm 4. We have the following guarantees. 

\begin{theorem}\label{thm:third_thm_Kusuoka}
Algorithm 4 run with the right parameters guarantees that with probability at least $1-\frac{1}{T}$
\begin{align*}
\bar{\mathcal{R}}^{\rho}_T \leq \tilde{O} (\frac{N^2 d^{16}}{\sqrt{T}}).
\end{align*}
\end{theorem}

\begin{theorem}\label{thm:fourth_thm_Kusuoka}
Let $f(\cdot ,\xi)$ be strongly convex with parameter $\beta>0$ for all $\xi \in \Xi$, Algorithm 4 run with the right parameters guarantees that with probability at least $1-\frac{3}{T}$
\begin{align*}
\mathcal{R}^{\rho}_T \leq \tilde{O}(\frac{N^3 d^8}{\beta^{1/2} T^{1/4}}).
\end{align*}
\end{theorem}

The proofs of these theorems can be found in the appendix.

\bibliography{mybib}
\bibliographystyle{abbrv} 
\newpage
\begin{appendices}

% !TEX root = aistats_main.tex

\section{More Preliminaries}\label{app_more_preliminaries}

\subsection{Convexity and Lipschitz Continuity}
Let $X\subseteq \mathbb{R}^d$ be a convex set, that is, for any $x,y\in X$ and any $\lambda \in [0,1]$, $\lambda x + (1-\lambda) y \in X$. We say $f:X\rightarrow \mathbb{R}$ is a convex function if for any $\lambda \in[0,1]$ and for any $x,y\in X$ 
\begin{align*}
 \lambda f(x) + (1-\lambda) f(y)  \geq f(\lambda x + (1-\lambda) y). 
\end{align*} 
An equivalent definition of convexity is the following \cite{nesterov2013introductory}. $f$ is convex if and only if 
\begin{align*}
f(x) \geq f(y) + \nabla f(y)^\top (x-y) \quad \forall x,y \in X.
\end{align*}
Here $\nabla f(y)$ denotes any element in the subdifferential  of $f$ at $y$.\\ 
We say $f:X\rightarrow \mathbb{R}$ is strongly convex with parameter $\beta > 0$ if and only if
\begin{align*}
f(x) \geq f(y) + \nabla f(y)^\top (x-y) + \frac{\beta}{2}||x-y||^2\quad \forall x,y\in X . 
\end{align*} 
We say $f$ is $G$-Lipschitz continuous with respect to a norm $||\cdot||$ if for every $x,y\in X$,
$|f(x)-f(y)| \leq G||x-y||$. 
\begin{lemma}\label{shalev_lipschitz}
\cite{shalev2012online} [Ch. 2]Let $f:X\rightarrow \mathbb{R}$ be a convex function. Then, $f$ is $G$-Lipschitz over $X$ with respect to a norm $||\cdot||$ if and only if for all $x \in X$ and for all $\nabla f(x) \in \partial f(x)$ we have that $||\nabla f(x) ||_* \leq G$, where $||\cdot||_*$ denotes the dual norm. 
\end{lemma}
Throughout this paper, whenever we say $f$ is $G$-Lipschitz we mean $f$ is $G$-Lipschitz with respect to $||\cdot||_2$ unless otherwise stated. 

\subsection {From OCO to to Bandit Feedback}
 We present a result from that allows us to transform regret bounds from OCO into expected regret bounds for Online Bandit Optimization.
 
\begin{lemma}\label{reduction_OBO}
\cite{hazan2016introduction}[Ch. 6] Let $u$ be a fixed point in $X$. Let $f_1, ..., f_T : X \rightarrow \mathbb{R}$ be a sequence of differentiable functions. Let $\mathcal{A}$ be a first order algorithm that ensures $\text{Regret}_T(\mathcal{A}) \leq B_{\mathcal{A}}(\nabla f_1(x_1),...,\nabla f_T(x_T)) $ in the full information setting. Define $\{x_t\}$ as: $x_1 \leftarrow \mathcal{A}(\emptyset)$ , $x_t\leftarrow \mathcal{A}(g_1,...,g_{t-1})$ where each $g_t$ satisfies:
 \begin{align*}
 \mathbb{E}[g_t| x_1, f_1, ..., x_t, f_t] = \nabla f_t(x_t) 
 \end{align*}
 Then, for every $u \in X$:
 \begin{align*}
 \mathbb{E}[\sum_{t=1}^T f_t(x_t)] - \sum_{t=1}^T f_t(u ) \leq \mathbb{E}[B_{\mathcal{A}}(g_1,...g_T)]
 \end{align*}
\end{lemma}
Moreover, Online Gradient Descent is a first order Algorithm \cite{hazan2016introduction}[Ch. 6]. 

\subsection{Some Useful Concentration Results }
In this section we present results on how quickly random functions uniformly concentrate around their mean. 
  
\begin{lemma}\cite{shalevstochastic}[Theorem 5] Let $\hat{F}(x) = \frac{1}{N}\sum_{n=1}^N f(x,\xi_n)$ where $f(\cdot, \xi)$ is $L$-Lipschitz with function values bounded by $R$ and the set where it is defined has diameter $B$. Let $F(x):=\mathbb{E}_{\xi}[f(x,\xi)]$. Then
\begin{equation} \label{eq:unif_conv_shalev}
P(\sup_{x\in X }|F(x)-\hat{F}(x)| \geq \epsilon) \leq O(d^2 (\frac{LB}{\epsilon})^d \exp(-\frac{N\epsilon^2}{128 L R})). 
\end{equation}
\end{lemma}
This result implies the following two lemmas. 

\begin{lemma}
\label{unif_convergence}
 With probability at least $1-\delta$, for any $x\in X$, over a sample size $N$ 
\begin{equation*} \label{eq:unif_conv_prob}
|F(x)-\hat{F}(x)| \leq \tilde{O} (\sqrt{\frac{LRd \ln(\frac{1}{\delta})}{N}}).
\end{equation*}
\end{lemma}

\begin{proof} Setting the right hand side of (\ref{eq:unif_conv_shalev}) equal to $\delta$ and solving for $\epsilon$ gives
\begin{align*}
\epsilon = \sqrt{\frac{128LR[2\ln(\frac{d}{\sqrt{\delta}}) + d \ln(LB) + d\ln(\frac{1}{\epsilon})]}{N}}
\end{align*}
Since we must bound $\epsilon$ by above, we now bound $\ln(\frac{1}{\epsilon})$. Using the previous equality we have
\begin{align*}
\ln(\frac{1}{\epsilon}) = \frac{1}{2} \ln(\frac{N}{128LR[2\ln(\frac{d}{\sqrt{\delta}}) + d \ln(\frac{LB}{\epsilon}) ]})
\end{align*}
since $\ln(\frac{LB}{\epsilon})$ is large and in the denominator, we have
\begin{align*}
\ln(\frac{1}{\epsilon}) \leq \frac{1}{2} \ln(\frac{N}{256LR\ln(\frac{d}{\sqrt{\delta}}) })
\end{align*}
this implies\footnote{Throughout the paper we let $\kappa$ be some universal constant that may change from line to line.}
\begin{align*}
\epsilon &\leq \sqrt{\frac{128LR[2\ln(\frac{d}{\sqrt{\delta}}) + d \ln(LB) + d\frac{1}{2} \ln(\frac{N}{256LR\ln(\frac{d}{\sqrt{\delta}}) }) ]}{N}}\\
&= \sqrt{\frac{\kappa LRd \ln(\frac{dLBN}{\sqrt{\delta} 256LR \ln(\frac{d}{\delta})})}{N}}\\
&= \tilde{O} (\sqrt{\frac{LRd \ln(\frac{1}{\delta})}{N}}) 
\end{align*}
\end{proof}

\begin{lemma}
\label{lemma_unif_conv_expectation}
\begin{equation*}\label{eq: unif_conv_expectation}
\mathbb{E}[\sup_{x\in X} |F(x)-\hat{F}(x)|] \leq \tilde{O} (\frac{\sqrt{LRd}}{\sqrt{N}}) 
\end{equation*}
\end{lemma}

\begin{proof}
Recall that for a nonnegative random variable $X$ it holds that $\mathbb{E}[X] =\int_{0}^\infty P(X>t) dt$. We have from (\ref{eq:unif_conv_shalev})
\begin{align*}
P(\sup_{x\in Z}|F(x) - \hat{F}(x)|>\epsilon) &\leq O (d^2 (\frac{LB}{\epsilon})^d \exp(-\frac{N\epsilon^2}{128LR}))\\
&= \exp[-(\frac{N\epsilon^2}{128LR} + d\ln(\epsilon) - 2\ln(d) - d\ln(LB))]
\end{align*}  
Let $\lambda(\epsilon) = a\epsilon^2 + d\ln(\epsilon)$ with $a:= \frac{N}{128LR}$ and notice that when $\epsilon \geq \sqrt{\frac{d}{2a}}$ the second derivative of $\lambda(\cdot)$ is nonnegative and therefore the function is convex in that domain thus we can lower bound it with its first order Taylor approximation at $\sqrt{\frac{d}{2a}}$. 
\begin{align*}
\lambda(\epsilon) \geq 2\sqrt{2ad}\epsilon-2d + \frac{d}{2} + \frac{d}{2}\ln(\frac{d}{2a})
\end{align*} 
Therefore, for $\epsilon\geq \sqrt{\frac{d}{2a}} $ 
\begin{align*}
P(\sup_{x\in Z}|F(x) - \hat{F}(x)|>\epsilon) &\leq  \exp[-(2\sqrt{2ad}\epsilon - 2d + \frac{d}{2} + \frac{d}{2}\ln(\frac{d}{2a})- 2\ln(d) -d\ln(LB) ) ]\\
&\leq \exp[-(2\sqrt{2ad}\epsilon - 2d + \frac{d}{2}\ln(\frac{d}{2a})- 2\ln(d) -d\ln(LB) ) ]\\
&\leq \exp[-(2\sqrt{2ad}\epsilon - 2d - \frac{d}{2}\ln(2a)- 2\ln(d) -d\ln(LB) ) ]\\
& = \exp[-(2\sqrt{2ad}\epsilon)+\theta]
\end{align*}
where $\theta := 2d + \frac{d}{2}\ln(2a)+2\ln(d) + d\ln(LB)$.
We have 
\begin{align*}
\mathbb{E}[\sup_{x\in X}|F(x) - \hat{F}(x)|] &\leq \int_{0}^\infty \min[1,\exp[-(2\sqrt{2ad}\epsilon)+\theta] ] d\epsilon\\
&= \int_{0}^{\epsilon'} d\epsilon+ \int_{\epsilon'}^\infty \exp[-2\sqrt{2ad}\epsilon + \theta] d\epsilon \quad \epsilon' = \frac{\theta}{2\sqrt{2ad}}\\
&= \epsilon' + \frac{\exp[\theta - 2\sqrt{2ad}\epsilon']}{2\sqrt{2ad}}\\
&= \frac{1}{2\sqrt{2ad}}[\theta + 1]\\
&= \frac{\sqrt{128LR}}{2\sqrt{2dN}} [ 2d + \frac{d}{2}\ln(2a)+2\ln(d) + d\ln(LB) +1�]\\
&= \tilde{O} (\frac{\sqrt{LRd}}{\sqrt{N}}) 
\end{align*}
\end{proof}

\subsection{Conditional Value at Risk}

\begin{proof}[Proof of Theorem \ref{concentration_cvar}] For any fixed $x \in X$, we define $\phi(z):= z + \frac{1}{\alpha} E_{\xi \sim P}[f(x,\xi) - z]_+ $ and $\widehat{\phi}(z) = \frac{1}{N} \sum_{n=1}^N z + \frac{1}{\alpha}[f(x,\xi_n)-z]_+ $. By Lemma \ref{unif_convergence} we know that with probability at least $1-\delta$ for all $z\in [0,1]$
\begin{align*}
|\phi(z)-\widehat{\phi}(z)| \leq O(\sqrt{\frac{L R \ln(N/\delta)}{N}})
\end{align*}
and it is easy to see that $L,R$ are both $O(\frac{1}{\alpha})$.

It remains to show that $A:=\{ X_A=\sup_{z}|\phi(z)-\widehat{\phi}(z)| \leq \epsilon \}$ implies $B:=\{X_B=|CVaR_\alpha[F](x) - \widehat{CVaR_\alpha[F](x)} | \leq \epsilon\}$. Indeed, we have that for any $z \in Z$
\begin{align*}
\phi(z)-\epsilon \leq \widehat{\phi}(z)
\end{align*}
Therefore, if $\bar{z} = \arg\min_{z \in Z} \widehat{\phi}(z)$ we have:
\begin{align*}
CVaR_{\alpha}[F](x) - \epsilon \leq \phi(\bar{z}) - \epsilon \leq \widehat{\phi}(\bar{z}) = \widehat{CVaR_{\alpha}}[F](x)
\end{align*}
The other side of the inequality follows by applying the same type of argument to $\widehat{\phi}(z) \leq \phi(z) + \epsilon$.
\end{proof}
\begin{remark}\label{remark_expectation}
We make one last remark about the proof above. We showed that $A \implies B$ therefore $P(B')\leq P(A')$. Since for a nonnegative random variable X we can write $\mathbb{E}[X]=\int P(X>\epsilon) d\epsilon$ we can conclude that $\mathbb{E}[X_B]\leq \mathbb{E}[X_A]$, or which is the same, $\mathbb{E}[|CVaR_\alpha[F](x) - \widehat{CVaR_\alpha[F](x)}|] \leq \mathbb{E}[\sup_{z}|\phi(z)-\widehat{\phi}(z)|]$.
\end{remark} 

\begin{lemma}\label{C_alpha_Lipschitz}
Let $\xi $ be a random variable supported in $\Xi$ with probability distribution $P$. Let $f:X\times \Xi \rightarrow \mathbb{R}$ and assume $0\leq f(x,\xi)\leq 1$ for all $x\in X$ and $\xi \in \Xi$. If $f(\cdot,\xi)$ is $G$-Lipschitz then so is $CVaR_{\alpha}[F](x)$. 
\end{lemma}

\begin{proof}
By Theorem 6.4 in \cite{shapiro2009lectures} for any $x\in X$. We have
\begin{align*}
CVaR_\alpha[F](x) = \sup_{\xi \in \Theta} \mathbb{E}_{\xi}[f(x,\xi)]
\end{align*}
where $\Theta$ is some family of probability distributions. 

Since convex combinations of $G$-Lipschitz functions is $G$-Lipschitz we have that for any $x_1 \in X$
\begin{align*}
\mathbb{E}_{\xi \in \Theta_1^*}[f(x_1,\xi)] - \mathbb{E}_{\xi \in \Theta_1^*}[f(x_2,\xi)] \leq G ||x_1-x_2||
\end{align*}
where $\Theta_1^*$ is the probability distribution that maximizes $\mathbb{E}_{\xi \in \Theta}[f(x_1,\xi)]$ (assuming it exists). Since 
\begin{align*}
\mathbb{E}_{\xi \in \Theta_1^*}[f(x_1,\xi)] - \mathbb{E}_{\xi \in \Theta_2^*}[f(x_2,\xi)] \leq \mathbb{E}_{\xi \in \Theta_1^*}[f(x_1,\xi)] - \mathbb{E}_{\xi \in \Theta_1^*}[f(x_2,\xi)]
\end{align*}
by combining the two inequalities we have 
\begin{align*}
CVaR_\alpha[F](x_1) - CVaR_\alpha[F](x_2) \leq G ||x_1-x_2||
\end{align*}
a symmetry argument yields the other side of the inequality, this concludes the proof.
\end{proof}

\begin{lemma}\label{lemma_projection}
Let $X$ be a convex set with diameter $D_{||\cdot||}$ that contains the origin, that is for all $x_1, x_2 \in X$, $ ||x_1-x_2||\leq D_{||\cdot||}$. Let $X_\delta:= \{x: x\in(1-\delta)X\}$. For any $x \in X$ let $x_\delta := \Pi_{X_\delta}(x)$ where the projection is taken with respect to any norm $||\cdot||$. Then
 \begin{align}
 ||x-x_\delta||\leq \delta D_{||\cdot||}
 \end{align}
\end{lemma}
\begin{proof}
Notice $(1-\delta)x \in X_\delta$
 \begin{align*}
 ||x-x_\delta||&\leq ||x-(1-\delta)x|| \quad \text{By definition of $\Pi$}\\
 &\leq \delta ||x||\\
 &\leq \delta D_{||\cdot||} \quad \text{since $X$ contains the origin}
 \end{align*} 
\end{proof}

\begin{lemma}\label{lemma_dual_norm}
Let $x = [x_1, x_2]^\top$. Define $||x|| = ||x_1||_2 + ||x_2||_\infty$. Then 
 \begin{align*}
 ||x||_* = \max\{||x_1||_2,||x_2||_1\}
 \end{align*} 
\end{lemma}

\begin{proof}
By definition of dual norm we have 
 \begin{align*}
 ||x||_* &= \max_{||y||\leq 1} x_1^\top y_1 + x_2^\top y_2\\
 &= \max_{||y_1||_2 + ||y_2||_\infty\leq 1} x^\top y\\
 &= \max_{c_1+c_2\leq 1} c_1 ||x_1||_2 + c_2 ||x_2||_1 \\
 & = \max\{||x_1||_2,||x_2||_1\}
 \end{align*}
\end{proof}

\subsection{Analysis of Algorithm 1}
\begin{lemma}
The function $\mathcal{L}_t(x,z) := z + \frac{1}{\alpha} [f_t(x) -z ]_+$ is jointly convex, $G_{\mathcal{L}}$-Lipschitz continuous with $G_{\mathcal{L}}= \alpha^{-1}(G +1) + 1$, and the diameter of the set where it is defined $D_{\mathcal{L}}\leq D_X + 1$.
\end{lemma}

\begin{proof}
We first prove convexity. The function $f_t(x) - z$ is jointly convex since both $f_t(x)$ and $-z$ are, and addition preserves convexity. Point-wise supremum over convex functions preserves convexity and since any constant function is convex we have that $[f_t(x)-z]_+$ is convex. Again, using the fact that addition preserves convexity we get the desired claim.

To prove the second part of the claim we notice: 
\begin{align*}
\nabla_x \mathcal{L}_t(x,z) &=
\begin{cases}
\frac{1}{\alpha}\nabla f_t(x) & \text{if } f_t(x)-z > 0 \\
0 & \text{otherwise}
\end{cases}\\
\nabla_z \mathcal{L}_t(x,z) &=
\begin{cases}
1 - \frac{1}{\alpha} & \text{if } f_t(x)-z > 0 \\
1 & \text{otherwise}
\end{cases}
\end{align*}
Let $\nabla \mathcal{L}_t := [\nabla_x \mathcal{L}_t ; \nabla_y \mathcal{L}_t ]$ and recall that a function $f$ is $G$-Lipschitz continuous if and only if $||\nabla f|| \leq G$. We have that
We have that
\begin{align*}
||\mathcal{L}_t|| &\leq \max\{||[\bar{0}; 1]||, ||[\alpha^{-1} \nabla f; 1 + \alpha^{-1} ]||\}\\
&\leq \alpha^{-1}(G +1) + 1=: G_{\mathcal{L}}
\end{align*}
Where the last inequality follows by simple algebra.

The fact that $D_{\mathcal{L}}\leq D_X + 1$ follows from the definition of the diameter of a set. 
\end{proof}
The key to prove Theorem  \ref{thm:pseudo_algo1} is to realize that Algorithm 1 is performing Online Gradient Descent using an estimate of the gradient of the smoothened function $\hat{\mathcal{L}}_t$ as in \cite{flaxman2005online}.

Next we prove a lemma assuming that for every $t=1,...,T$ $\nabla \mathcal{L}_t:=\nabla \mathcal{L}_t(x_t,z_t)$ is revealed and we update according to 
\begin{align}
[x_{t+1},z_{t+1}]^\top \leftarrow \Pi_{X\times Z} ([x_{t},z_{t}]^\top - \eta \nabla \mathcal{L}_t)
\end{align}
That is, we perform Zinkevich's Online gradient Descent (OGD) on functions $\mathcal{L}_t$ \cite{zinkevich2003online}. Due to Lemma \ref{reduction_OBO} we will be able to use this guarantee when we have bandit feedback.

\begin{lemma}
Applying OGD on sequence of functions $\{\mathcal{L}_t\}_{t=1}^T$  guarantees: for every $w=(x,z) \in \mathcal{W}:= X \times Z$. 
\begin{align*}
\sum_{t=1}^T \mathcal{L}_t(w_t) - \sum_{t=1}^T \mathcal{L}_t(w) \leq \frac{D_{\mathcal{L}}}{2\eta} +  \frac{\eta}{2} \sum_{t=1}^T ||\nabla \mathcal{L}_t||^2.
\end{align*}
\end{lemma}

\begin{proof}
We follow Zinkevich's proof. By properties of projections we have:
\begin{align*}
||w_{t+1}-w||^2 & \leq ||w_{t} - \eta \nabla \mathcal{L}_t -w||^2 \\
&= ||w_{t}-w||^2 + \eta^2 ||\nabla \mathcal{L}_t||^2 - 2 \eta \nabla \mathcal{L}_t ^\top (w_t - w ) 
\end{align*}
Therefore: 
\begin{align*}
2 \eta \nabla \mathcal{L}_t ^\top (w_t - w) \leq \frac{||w_t-w||^2 - ||w_{t+1}-w||^2}{\eta} +\eta || \nabla \mathcal{L}_t||^2
\end{align*}
Using convexity and summing up the inequalities above for every $t$ we have: 
\begin{align} 
2(\sum_{t=1}^T \mathcal{L}_t(w_t) - \sum_{t=1}^T \mathcal{L}_t(w)) &\leq  \sum_{t=1}^T 2 \eta \nabla \mathcal{L}_t ^\top (w_t - w)  \label{ogd_telescoping1} \\
& \leq \sum_{t=1}^T \frac{||w_t-w||^2 - ||w_{t+1}-w||^2}{\eta} + \eta \sum_{t=1}^T ||\nabla \mathcal{L}_t||^2 \label{ogd_telescoping2} \\
& \leq \frac{D_{\mathcal{L}}}{\eta} +  \eta \sum_{t=1}^T ||\nabla \mathcal{L}_t||^2 \nonumber
\end{align}
 Which yields the desired result.
 \end{proof}

\begin{lemma}\label{algo1_L}
Let $\tilde{y}_t=(\tilde{x}_t, \tilde{z}_t)$ and $y^* = (x^*,z^*) := \text{argmin}_{x,z\in X\times Z} \sum_{t=1}^T \mathbb{E}_{\xi}[\mathcal{L}_t (x,z)] $, Algorithm 1 guarantees: 
\begin{align*}
\sum_{t=1}^T \mathbb{E}_{int} [\mathcal{L}_t (\tilde{y}_t) ]- \sum_{t=1}^T \mathcal{L}_t(y^*) = O(\frac{dD_X G T^{3/4}}{\alpha}) 
\end{align*}
\end{lemma}

\begin{proof}
Define $y^*_\delta = \Pi_{X_\delta}[y^*]$. By Lemma \ref{lemma_projection} in the Appendix, it holds that $||y^*_\delta - y^*||\leq \delta D_{\mathcal{L}}$. 
Using a similar argument as in \cite{flaxman2005online} we have:
\begin{align*}
&\mathbb{E}_{int}[\sum_{t=1}^T \mathcal{L}_t(\tilde{y}_t) - \sum_{t=1}^T \mathcal{L}_t(y^*)]\\ 
& \leq \mathbb{E}_{int}[\sum_{t=1}^T \mathcal{L}_t(y_t) - \sum_{t=1}^T \mathcal{L}_t(y^*)] +  \delta G_{\mathcal{L}} T \quad \text{by Lemma \ref{smooth_f} and $||y_t-\tilde{y}_t||\leq \delta$}\\ � 
& \leq \mathbb{E}_{int}[\sum_{t=1}^T \mathcal{L}_t(y_t) - \sum_{t=1}^T \mathcal{L}_t(y_{\delta}^*)] +  \delta G_{\mathcal{L}} T + \delta G_{\mathcal{L}}D_{\mathcal{L}} T\\
& \leq \mathbb{E}_{int}[\sum_{t=1}^T \hat{\mathcal{L}}_t(y_t) - \sum_{t=1}^T \hat{\mathcal{L}}_t(y_{\delta}^*)] +  3\delta G_{\mathcal{L}} T + \delta G_{\mathcal{L}}D_{\mathcal{L}} T \quad \text{by Lemma \ref{smooth_f}}\\
& \leq \frac{\eta}{2} \sum_{t=1}^T \mathbb{E}_{int} [||g_t||^2] + \frac{D_{\mathcal{L}}^2 }{2\eta} + 3\delta G_{\mathcal{L}} T+  \delta D_{\mathcal{L}} G_{\mathcal{L}}T\quad \text{by Lemma \ref{reduction_OBO}}\\
&\leq \frac{\eta}{2} \frac{(d+1)^2}{\delta^2} \sum_{t=1}^T |\tilde{z}_t + \frac{1}{\alpha}[f_t(\tilde{x}_t) - \tilde{z}_t]|^2 +\frac{D_{\mathcal{L}}^2}{2 \eta} + 3\delta G_{\mathcal{L}} T + \delta D_{\mathcal{L}} G_{\mathcal{L}}T\\
&\leq \frac{\eta}{2} \frac{(d+1)^2}{\delta^2 \alpha^2} T + \frac{D_{\mathcal{L}}^2}{2\eta} +3\delta G_{\mathcal{L}} T + \delta D_{\mathcal{L}}G_{\mathcal{L}} T\\
&= O(\frac{dD_X G T^{3/4}}{\alpha})
\end{align*}
Where we chose $\eta = O(\frac{D_X \alpha}{d T^{3/4}})$ and $\delta = O(\frac{1}{T^{1/4}})$.
\end{proof}

We are now ready to give a proof of Theorem \ref{thm:pseudo_algo1}.
\begin{proof}[ Proof of Theorem \ref{thm:pseudo_algo1}] Notice that for all $t$, every $x \in X$ and every $z \in Z$, we have:
 \begin{align*}
 \mathbb{E}_{\xi\sim P}[\mathcal{L}_t(x,z)] = z + \frac{1}{\alpha}\mathbb{E}_{\xi \sim P}[f(x,\xi) - z]_+ \geq CVaR_{\alpha}[F](x).
 \end{align*}
The result then follows by taking $\mathbb{E}_{\xi \sim P} [\cdot]$ in both sides of the result in Lemma \ref{algo1_L} and interchanging the expectations. The interchange can be done using Fubini's Theorem since for every $x\in X$ and for every $z\in Z$ we have that $\mathcal{L}_t(x,z) <O (\frac{1}{\alpha})$ almost surely.
\end{proof}

We are now ready to prove Theorem \ref{thm:regret_algo1}.
 We assume $f_t$ is $1$-Lipschitz continuous.

\begin{proof} [Proof of Theorem \ref{thm:regret_algo1}]
Define concentration error $CE = C_\alpha[\{f_t(x^*)\}_{t=1}^T] -C_\alpha[\{f_t(\bar{x})\}_{t=1}^T ]$, where $\bar{x} = \arg\min_{x\in X}C_\alpha[\{f_t(x)\}_{t=1}^T ] $, let $x^* = \arg\min_{x\in X} C_\alpha�[F](x)$, we have

 \begin{align*}
&  \mathbb{E}[ C_\alpha[\{f_t(x_t)\}_{t=1}^T] \pm  C_\alpha [\{f_t(x^*)\}_{t=1}^T] ] - min_{x \in X} C_{\alpha}[\{f_t(x)\}_{t=1}^T] \\
&= \mathbb{E} [ \min_y y + \frac{1}{\alpha T} \sum_{t=1}^T \max\{f_t(x_t)+f_t(x^*) -f_t(x^*) - y,0\}  - C_\alpha [\{f_t(x^*)\}_{t=1}^T] ] + \mathbb{E}[CE]\\
&\leq \mathbb{E} [ \min_y y + \frac{1}{\alpha T} \sum_{t=1}^T \max\{f_t(x^*)+ |f_t(x_t)-f_t(x^*)| - y,0\}  -C_\alpha [\{f_t(x^*)\}_{t=1}^T] ] +\mathbb{E}[CE]\\
&\leq \mathbb{E} [ \min_y y + \frac{1}{\alpha T} \sum_{t=1}^T \max\{f_t(x^*)+ |f_t(x_t)-f_t(x^*)| - y,|f_t(x_t)-f_t(x^*)|\}  -C_\alpha [\{f_t(x^*)\}_{t=1}^T] ] + \mathbb{E}[CE]\\
&= \mathbb{E} [ \min_y y + \frac{1}{\alpha T} \sum_{t=1}^T \max\{f_t(x^*) - y,0\} + \frac{1}{\alpha T} \sum_{t=1}^T + |f_t(x_t)-f_t(x^*)| -C_\alpha [\{f_t(x^*)\}_{t=1}^T] ] + \mathbb{E}[CE]\\
&= \mathbb{E} [  \frac{1}{\alpha T} \sum_{t=1}^T |f_t(x_t) - f(x^*) |  ] + \mathbb{E}[CE]\\
&\leq \frac{1}{\alpha T} \sum_{t=1}^T \mathbb{E}_t [ ||x_t - x^*|| ] +\mathbb{E}[CE] \quad \text{since $f_t$ is $1$-Lipschitz}\\
&\leq \frac{1}{\alpha T}  \sqrt{T} \sqrt{  \sum_{t=1}^T \mathbb{E}_t[ ||x_t - x^*||] ^2 } + \mathbb{E}[CE] \quad \text{by Cauchy Schwartz}\\
&\leq \frac{1}{\alpha T}  \sqrt{T} \sqrt{  \sum_{t=1}^T \mathbb{E}_t[ \frac{2}{\beta}[C_\alpha[F](x_t) -C_\alpha[F](x^*)]] } + \mathbb{E}[CE] \quad \text{by strong convexity of $C_\alpha[F](\cdot)$ and KKT condition} \\
& = \frac{1}{\alpha T}  \sqrt{T} \sqrt{\frac{2}{\beta} \mathbb{E}[ \sum_{t=1}^T C_\alpha[F](x_t) -C_\alpha[F](x^*)  ]} + \mathbb{E}[CE]\\
& = O (\frac{d^{1/2}}{\alpha^{3/2} \beta^{1/2} T^{1/8}}  ) + \mathbb{E}[CE] \quad \text{by Theorem \ref{thm:pseudo_algo1}}
\end{align*}
We still need to bound the concentration error $CE$ in expectation. Notice we can write 
\begin{align*}
 CE = [C_\alpha[\{f_t(x^*)\}_{t=1}^T] - C_{\alpha}[F](x^*)] + [C_{\alpha}[F](x^*) - C_{\alpha}[F](\bar{x})] + [C_{\alpha}[F](\bar{x}) -C_\alpha[\{f_t(\bar{x})\}_{t=1}^T ]
 \end{align*}

and the second term is nonpositive. To bound $CE$ in expectation we apply Lemma \ref{lemma_unif_conv_expectation} on functions $\phi(x,y)= y + \frac{1}{\alpha}[f(x)-y]_+$ (notice $L \leq O(\frac{1}{\alpha})$ and $R= O (\frac{1}{\alpha})$), by Remark \ref{remark_expectation} and the same reasoning as in the proof of Lemma \ref{lemma_unif_conv_expectation} we have $\mathbb{E}[|C_{\alpha}[F](\bar{x}) -C_\alpha[\{f_t(\bar{x})\}_{t=1}^T|]\leq \tilde{O}( \frac{\sqrt{d}}{\alpha \sqrt{T}})$. Thus $\mathbb{E}[CE] \leq  \tilde{O}( \frac{\sqrt{d}}{\alpha \sqrt{T}})$. This finishes the proof. 
\end{proof}
%%%%%%%%%%%%%%%%%%%END ALGORITHM 1%%%%%%%%%%%%%%%%%%%%%%

%%%%%%%%%%%%%%%%%%%%% ANALYSIS ALGORITHM 2 1-D%%%%%%%%%%%%%%%%
\subsection{Analysis of Algorithm 2 (1-D)}
We proceed to formally analyze the algorithm following \cite{agarwal2011stochastic}. In this section, for ease of reading we refer to quantity $T\bar{\mathcal{R}}_T$ as the regret. We work conditioned on $\mathcal{E}$ which is defined as the event that for every epoch and for every round $i$, $h(x) \in [LB_{\gamma_i}(x), UB_{\gamma_i}(x)]$ for $x\in \{x_l, x_c,x_r\}$. We will first bound the regret in an epoch and then bound the total number of epochs. We do the previous in the next sequence of lemmas. Notice that by Theorem \ref{concentration_cvar} we can obtain a $\gamma$-CI for $h(x)$ that holds with probability at least $1-\frac{1}{T^2}$ with only $\frac{\kappa \ln(T/(\alpha \gamma))}{\alpha^2 \gamma^2}$ samples. We first show that we never discard points that are near optimal. 

\begin{lemma}\label{lemma_never_throw_x}
If epoch $\tau$ ends in round $i$, then the interval $[l_{\tau+1},r_{\tau+1}]$ contains every $x\in [l_{\tau}, r_{\tau}]$ such that $h(x) \leq h(x^*) + \gamma_i$. In particular, $x^*\in[l_{\tau},r_{\tau}]$ for all epochs $\tau$.
\end{lemma}

\begin{proof}
Assume epoch $\tau$ terminates in round $i$ through Case 1. Then, either $LB_{\gamma_i}(x_l) \geq UB_{\gamma_i}(x_r) + \gamma_i$ or $LB_{\gamma_i}(x_r) \geq UB_{\gamma_i}(x_l) + \gamma_i$. We assume the former occurs. It then holds that
\begin{align*}
h(x_l)\geq h(x_r) + \gamma_i.
\end{align*}

We must show that the points in the working feasible region to the left of $x_l$ are not near optimal. That is, for every $x\in [l_{\tau}, l_{\tau+1}] = [l_{\tau}, x_l]$ we have $h(x) \geq h(x^*)+\gamma_i$. Pick $x\in[l_{\tau},x_l]$ then, for some $t\in[0,1]$ we have $x_l = t x+(1-t)x_r$. Since $h$ is convex we have
\begin{align*}
h(x_l)\leq t h(x)+ (1-t) h(x_r)
\end{align*}
which implies
\begin{align*}
h(x) &\geq h(x_r) + \frac{h(x_l)-h(x_r)}{t}\\
&\geq h(x_r) + \frac{\gamma_i}{t}\\
&\geq h(x^*) + \gamma_i
\end{align*}
as required. If $LB_{\gamma_i}(x_r) \geq UB_{\gamma_i}(x_l) + \gamma_i$ had occurred the argument is analogous.

If epoch $\tau$ had terminated through case 2 then
\begin{align*}
\max\{LB_{\gamma_i}(x_l),LB_{\gamma_i}(x_r)\} \geq UB_{\gamma_i}(x_c)+ \gamma_i.
\end{align*}
We assume $LB_{\gamma_i}(x_l) \geq UB_{\gamma_i}(x_c)+ \gamma_i$, then 
\begin{align*}
h(x_l)\geq h(x_c) + \gamma_i.
\end{align*}
The same argument as above with $x_c$ instead of $x_r$ guarantees $h(x_l)\geq h(x^*) + \gamma_i$. If $LB_{\gamma_i}(x_r) \geq UB_{\gamma_i}(x_c)+ \gamma_i$ had occurred the argument is analogous. The fact that $x^*\in[l_{\tau},r_{\tau}]$ for every epoch $\tau$ follows by induction.
\end{proof}

We now show that if an epoch does not terminate in a given round $i$ then the regret ($T\bar{\mathcal{R}}_T$) incurred in that epoch was not too high.

\begin{lemma}
If epoch $\tau$ continues from round $i$ to $i+1$ then the regret in round $i$ is at most
\begin{align*}
\frac{\kappa \ln(T/(\alpha \gamma_i))}{\alpha^2 \gamma_i}
\end{align*}
\end{lemma}

\begin{proof} 
The regret incurred in round $i$ of epoch $\tau$ is
\begin{align*}
\frac{\kappa \ln(T/(\alpha \gamma_i))}{\alpha^2 \gamma_i^2}[(h(x_l)-h(x^*)) + (h(x_c)-h(x^*))+ (h(x_r)-h(x^*))]
\end{align*}
It suffices to show that for every $x\in \{x_l,x_c,x_r\}$ it holds that
\begin{align*}
h(x)\leq h(x^*) + 12\gamma_i.
\end{align*}
The algorithm continues from round $i$ to round $i+1$ if and only if
\begin{align*}
\max\{LB_{\gamma_i}(x_l),LB_{\gamma_i}(x_r)\} < \min\{UB_{\gamma_i}(x_l),UB_{\gamma_i}(x_r)\} + \gamma_i
\end{align*}
and
\begin{align*}
\max\{LB_{\gamma_i}(x_l), LB_{\gamma_i}(x_r)\} < UB_{\gamma_i}(x_c) + \gamma_i.
\end{align*}

This implies that $h(x_l), h(x_c),$ and $h(x_r)$ are all contained in an interval of at most $3 \gamma_i$. There are two cases for which the argument is essentially the same, either $x^* \leq x_c$ or $x^* > x_c$, we consider the former. Since by the previous lemma we know that $x^*\in[l_{\tau},r_{\tau}]$, then there exists $t\in[0,1]$ such that $x^*=x_c + t(x_c - x_r)$. Therefore 
\begin{align*}
x_c = \frac{1}{1+t}x^* + \frac{t}{1+t} x_r.
\end{align*}
Since $|x_c - l_{\tau}|=w_{\tau}/2$ and $|x_r-x_c|=w_{\tau}/4$ we have
\begin{align*}
t= \frac{|x^*-x_c|}{|x_r-x_c|} \leq \frac{| l_{\tau}-x_c|}{|x_r - x_c|} = \frac{w_{\tau}/2}{w_{\tau}/4}=2
\end{align*}
Since $h$ is convex
\begin{align*}
h(x_c) \leq \frac{1}{1+t}h(x^*)+\frac{t}{1+t}h(x_r)
\end{align*}
therefore
\begin{align*}
h(x^*) &\geq (1+t)\big( h(x_c)-\frac{t}{1+t}h(x_r)\big)\\
&= h(x_c) + (1+t)(h(x_c)-h(x_r))\\
&\geq h(x_c) - (1+t) |h(x_c)-h(x_r)|\\
&\geq h(x_r) - (1+t) 3\gamma_i\\
&\geq h(x_r) - 9\gamma_i
\end{align*}
So, for all $x\in\{x_l,x_c,x_r\}$ it holds that
\begin{align*}
h(x)\leq h(x_r) + 3\gamma_i \leq h(x^*)+12\gamma_i.
\end{align*}
\end{proof}

We proceed to bound the regret in each epoch.

\begin{lemma} 
If epoch $\tau$ ends in round $i$ the regret incurred in the epoch is no more than
\begin{align*}
\frac{\kappa \ln(T/(\alpha \gamma_i))}{\alpha^2 \gamma_i}.
\end{align*}
\end{lemma}

\begin{proof}
If $i=1$, since $h(x)$ is $1$-Lipschitz and $X=[0,1]$ we have that for every $x\in\{x_l,x_c,x_r\}$ $h(x)-h(x^*)\leq1$. Therefore the regret in epoch $\tau$ is 
\begin{align*}
&\frac{\kappa \ln(T/(\alpha^2 \gamma_i^2))}{\alpha^2 \gamma_i^2} \big( ((h(x_l)-h(x^*) )  + (h(x_c)-h(x^*)) + (h(x_r)-h(x^*)) \big) \\
& \leq \frac{6 \kappa \ln(T/(\alpha^2 \gamma_i^2))}{\alpha^2 \gamma_1}
\end{align*}
If $i\geq 2$, by the previous lemma we have that the regret incurred in round $j$ with $1\leq j \leq i-1$ is no more than
\begin{align*}
\frac{\kappa \ln(T/(\alpha^2 \gamma_i^2))}{\alpha^2 \gamma_j}.
\end{align*} 
For round $i$ the regret incurred is at most
\begin{align*}
3\cdot12\gamma_{i-1} \frac{\kappa \ln(T/(\alpha^2 \gamma_i^2))}{\alpha^2 \gamma_i^2}  =  \frac{\kappa 72 \ln(T/(\alpha^2 \gamma_i^2))}{\alpha^2 \gamma_i}.
\end{align*} 
It follows that the regret in epoch $\tau$ is
\begin{align*}
&\sum_{j=1}^{i-1} \frac{\kappa \ln(T/(\alpha^2 \gamma_j^2 ))}{\alpha^2 \gamma_j} + \frac{\kappa \ln(T/(\alpha^2 \gamma_i^2))}{\alpha^2 \gamma_i}\\
&= \sum_{j=1}^{i-1} \frac{\kappa \ln(T/(\alpha^2 \gamma_j^2 ))}{\alpha^2 }\cdot 2^j + \frac{\kappa \ln(T/(\alpha^2 \gamma_i^2))}{\alpha^2 \gamma_i}\\
& < \frac{\kappa \ln(T/(\alpha^2 \gamma_i^2 ))}{\alpha^2 } \cdot 2^i + \frac{\kappa \ln(T/(\alpha^2 \gamma_i^2))}{\alpha^2 \gamma_i} \\
&=\frac{\kappa \ln(T/(\alpha \gamma_i))}{\alpha^2 \gamma_i}.
\end{align*}
\end{proof}

We have bounded the regret that we incur in each epoch. We proceed to bound the number of epochs. 

\begin{lemma}
The total number of epochs $\tau$ satisfies
\begin{align*}
\tau \leq \kappa \log_{4/3}(\frac{\alpha^2 T}{\ln(T)}).
\end{align*}
\end{lemma}
\begin{proof}
The key is to observe that since the number of times we sample a point is bounded above by $T$ then $\gamma_i \geq (\alpha^2 T / (\kappa \ln(T)))^{-1/2}$ for every round and every epoch. Let $\gamma_{min}:= (\alpha^2 T / (\kappa \ln(T)))^{-1/2}$ and let $I := [x^* - \gamma_{min},x^* + \gamma_{min}] $. Since $h$ is $1$-Lipschitz, for any $x\in I$
\begin{align*}
h(x)-h(x^*)\leq \gamma_{min}.
\end{align*}
By Lemma \ref{lemma_never_throw_x} we have that for any round $\tau'$ which ends in round $i'$ 
\begin{align*}
I \subseteq \{x\in[0,1]: f(x)<f(x^*)+\gamma_{i'} \} \subseteq [l_{\tau'+1}, r_{\tau'+1} ]
\end{align*}
since $\gamma_{min}\leq \gamma_{i'}$. The previous implies
\begin{align*}
2\gamma_{min}\leq r_{\tau+1} - l_{\tau+1} = w_{\tau+1}.
\end{align*}
By the definitions of $l_{\tau'+1}$, $r_{\tau'+1}$ and $w_{\tau'+1}$ we have that for any $\tau' \in \{1,...,\tau\}$
\begin{align*}
w_{\tau'+1}\leq \frac{3}{4}w_{\tau'}.
\end{align*}
Therefore,
\begin{align*}
2\gamma_{min} \leq w_{\tau+1}\leq (\frac{3}{4})^{\tau}w_1 \leq (\frac{3}{4})^{\tau}
\end{align*}
which yields the result. 
\end{proof}

We are now ready to prove Theorems \ref{thm:first_theorem_algo2_1d} and \ref{thm:second_theorem_algo2_1d} .

\begin{proof}[ Proof of Theorem \ref{thm:first_theorem_algo2_1d}]
The per epoch regret when epoch $\tau$ ends in round $i$ is
\begin{align*}
\frac{\kappa \ln(T/(\alpha \gamma_i))}{\alpha^2 \gamma_i} \leq \frac{\kappa \ln(T/(\alpha \gamma_i))}{\alpha^2 \gamma_{\min}} \leq \frac{\kappa \sqrt{T} \ln(T/(\alpha \gamma_{\min}))}{\alpha} = \frac{\kappa \sqrt{T}\ln(T)}{\alpha}.
\end{align*}
Using the previous lemma we know that the regret will not be more than 
\begin{align*}
\frac{\kappa \sqrt{T}\ln(T)}{\alpha}  \log_{4/3}(\frac{\alpha^2 T}{\ln(T)})
\end{align*}
Recall we have been working conditioned on $\mathcal{E}$. We need an upper bound on $P( \mathcal{E}' )$. We know that after $\frac{\kappa \ln(T/ (\alpha \gamma))}{\alpha^2 \gamma_i}$ queries we have
\begin{align*}
P(|\hat{h}(x)-h(x)|\geq \gamma_i) \leq \frac{1}{T^2}.
\end{align*}
Since there are at most $T$ epochs a union bound gives
\begin{align*}
P(\mathcal{E}')\leq \frac{1}{T}
\end{align*}
which yields the desired result.
\end{proof}

\begin{proof} [ Proof of Theorem \ref{thm:second_theorem_algo2_1d}]
The proof is very similar to that of Theorem  \ref{thm:regret_algo1} with the difference that we have to bound the concentration error $CE:= C_\alpha[\{f_t(x^*)\}_{t=1}^T] - \min_{x \in X}C_\alpha[\{f_t(x)\}_{t=1}^T] $ with high probability. As explained in the proof of Theorem \ref{thm:regret_algo1} we know
\begin{align*}
CE \leq |C_\alpha[\{f_t(x^*)\}_{t=1}^T] - C_\alpha [F] (x^*) | + |C_\alpha[F](\bar{x}) - C_\alpha[\{f_t(\bar{x})\}_{t=1}^T]|
\end{align*}
where $\bar{x} = \arg\min_{x\in X} C_\alpha[\{f_t(x)\}_{t=1}^T]$. To bound $CE$ with high probability we apply Lemma \ref{unif_convergence} with $\delta = 1/T$ on functions $\phi(x,y)= y + \frac{1}{\alpha}[f(x)-y]_+$ (notice $L \leq O(\frac{1}{\alpha})$ and $R= O (\frac{1}{\alpha})$), by the same reasoning as in the proof of Theorem \ref{concentration_cvar} we have that with probability at least $1-\frac{1}{T}$, $|C_\alpha[F](\bar{x}) - C_\alpha[\{f_t(\bar{x})\}_{t=1}^T] | \leq \tilde{O}(\frac{1}{\alpha \sqrt{T}})$ and thus by a union bound we have that with probability at least $1-\frac{2}{T}$, $CE \leq \tilde{O}(\frac{1}{\alpha \sqrt{T}})$. As in the proof of Theorem \ref{thm:regret_algo1} we have 
\begin{align*}
\mathcal{R}_T \leq \frac{\sqrt{T}}{\alpha T \beta^{1/2}} \sqrt{T\bar{\mathcal{R}}_T} + CE.
\end{align*}
Using Theorem \ref{thm:first_theorem_algo2_1d} to bound $\bar{\mathcal{R}}_T$, the argument in the previous paragraph to bound $CE$, and a union bound yields the result.
\end{proof}
%%%%%%%%%%%%%%%%%%%%%END ANALYSIS ALGORITHM 2 1-D%%%%%%%%%%%%

%%%%%%%%%%%%%%%%%%%%% ANALYSIS OF ALGORITHM 2 d-D%%%%%%%%%%%%
\subsection{Analysis of Algorithm 2 ($d$-D) }

We first  describe the algorithm informally. As in the special case from the previous section, Algorithm 2 proceeds in epochs. Let the initial working feasible region be $\mathcal{X}_0=X$. The goal is that at the end of every epoch $\tau$ we will discard some portion of the working region $\mathcal{X}_{\tau}$ and end up with a smaller region $\mathcal{X}_{\tau +1}$ which contains at least one approximate optimum. 

We now give a brief description of the algorithm. At the beginning of every epoch $\tau$ we apply an affine transformation to the current working region $\mathcal{X}_\tau$ such that the smallest ellipsoid that contains it is an Euclidean ball of radius $R_{\tau}$ which we denote $\mathcal{B}(R\tau)$. We assume that $R_1\leq 1. $ Let $r_\tau := R_\tau / (c_1d)$ for some $c_1\geq 1$ so that $\mathcal{B}(r_\tau )\subseteq \mathcal{X}_\tau$ (such a construction is always possible see Lecture 1 p. 2 of  \cite{ball1997elementary} ). We refer to the enclosing ball $\mathcal{B}(R_\tau)$ as $\mathcal{B}_{\tau}$. Every epoch will consist of several rounds where $\gamma_i$ is halved in every round. \\
Let $x_0$ be the center of $\mathcal{B}_\tau$. At the start of  epoch $\tau$, we build a simplex with center $x_0$ contained in $\mathcal{B}(r_\tau)$. We will play the vertices of the simplex $x_1,....,x_{d+1}$ enough times so that the CI's at each vertex are of width $\gamma_i$ and hold with high probability. The algorithm will then choose point $y_1$ for which $\hat{h}(x)_i$ is the largest, here $\hat{h}$ denotes the empirical estimate of $h$. By construction we are guaranteed that $h(y_1)\geq h(x_j)-\gamma_i$ for $j=1,...,d+1$.\\
The algorithm will now try to identify a region where the function value is high so that at the end of the epoch we can discard it. It will do this by constructing pyramids with parameter $\hat{\gamma}$ (always greater that $\gamma$) until a bad region is found, if this does not happen for the current value of $\gamma$ it means that the algorithm did not incur to much regret (relative to how large $\gamma$ was). The pyramid construction follows from Section 9.2.2 of \cite{nemirovskii1983problem}. The pyramids have angle $2\phi$ at the apex where $\cos(\phi)= c_2/d$.  The base of the pyramid has $d$ vertices, $z_1,...,z_d$ such that $z_i-x_0$ and $y_1-z_i$ are orthogonal. The previous construction is always possible. Indeed, take a sphere with diameter $y_1-x_0$ and arrange $z_1,...z_d$ on its boundary such that the angle between $y_1-x_0$ and $y_1-z_i$ is $\phi$. We now set $\hat{\gamma} = 1$ and play all the points $y_1,z_1,...z_d$, and the center of the pyramid enough times until all the CI's are of width $\hat{\gamma}$. Let \textsc{top} and \textsc{bottom} be the vertices of the pyramid (including $y_1$) with the largest and smallest values for $\hat{h}(x)$. Let $\Delta(\cdot), \bar{\Delta}(\cdot)$, be functions which are specified later. We then check for one of the following cases:
\begin{enumerate}
\item If $LB_{\hat{ \gamma}} (\textsc{top}) \geq UB_{\hat {\gamma}}(\textsc{bottom}) + \Delta_{\tau}(\hat{\gamma})$ then we proceed depending on what the separation between the CI's of \textsc{top} and \textsc{apex} is.
\begin{enumerate}

\item If $LB_{\hat{\gamma}}(\textsc{top}) \geq UB_{\hat{\gamma}}( \textsc{apex}) +\hat{\gamma}$, then with high probability 
\begin{align*}
h(\textsc{top}) \geq h(\textsc{apex})+\hat{\gamma} \geq h(\textsc{apex}) + \gamma_i.
\end{align*}
We then build a new pyramid with apex equal to \textsc{top}, reset $\hat{\gamma}=1$ and continue sampling on the new pyramid.
\item If $LB_{\hat{\gamma}}(\textsc{top}) < UB_{\hat{\gamma}}( \textsc{apex}) +\hat{\gamma}$, then $LB_{\hat{\gamma}}(\textsc{apex})\geq UB_{\hat{\gamma}}(\textsc{bottom})+\Delta(\hat{\gamma})-2\hat{\gamma}.$ We then conclude the epoch and pass the current apex to the cone-cutting subroutine. 
\end{enumerate}

\item If $LB_{\hat{ \gamma}} (\textsc{top}) < UB_{\hat {\gamma}}(\textsc{bottom}) + \Delta_{\tau}(\hat{\gamma})$, then one of the following things happen:
\begin{enumerate}
\item If $UB_{\hat{\gamma}}(\textsc{center}) \geq LB_{\hat{\gamma}}(\textsc{bottom})- \bar{\Delta}_{\tau}(\hat{\gamma})$, then all the vertices of the pyramid and the center of the pyramid have function values in an interval of size $2\Delta_{\tau}(\hat{\gamma})+3\hat{\gamma}$.  We can then set $\hat{\gamma}=\hat{\gamma}/2$. If $\hat{\gamma}<\gamma_i$, we start the next round with $\gamma_{i+1}=\gamma_{i}/2$. Otherwise we continue sampling with the new $\hat{\gamma}$. 
\item If $UB_{\hat{\gamma}}(\textsc{center}) < LB_{\hat{\gamma}}(\textsc{bottom})- \bar{\Delta}_{\tau}(\hat{\gamma})$. We conclude the epoch and pass the center and current apex to the hat-raising subroutine. 
\end{enumerate}
\end{enumerate}

\textbf{Hat-Raising:} This occurs whenever the pyramid satisfies $LB_{\hat{\gamma}}(\textsc{top})\leq UB_{\hat{\gamma}}(\textsc{bottom}) + \Delta_{\tau}(\hat{\gamma})$ and $UB_{\hat{\gamma}}(\textsc{cent})\leq LB_{\hat{\gamma}}(\textsc{bottom})-\bar{\Delta}_{\tau}(\hat{\gamma})$. We will later show that if we move the apex a little from $y_i$ to $y_i'$, then the CI of $y_i'$ is above the CI of \textsc{top} and the new angle $\phi'$ in not too much smaller than $2\phi$. In particular, we will let $y_i'=y_i + (y_i - \textsc{center}_i)$.

\textbf{Cone-cutting:} This is the last step in a given epoch (notice this is the last step in the hat-raising subroutine). This subroutine receives a pyramid with apex $y$ and base $z_1,...,z_d$ with angle $2\bar{\phi}$ at the apex such that $\cos(\bar{\phi})\leq 1/2d$. Define the cone
\begin{align}
K_{\tau} = \{ x: \exists \lambda > 0, \alpha_1,...,\alpha_d > 0, \sum_{i=1}^d \alpha_i=1 : x = y - \lambda \sum_{i=1}^d \alpha_i (z_i-y) \}
\end{align}
which is centered at $y$ and is the reflection of the pyramid around the apex. By construction $\mathcal{K}_{\tau}$ has angle $2\bar{\phi}$ at the apex. Let $\mathcal{B}_{\tau+1}'$ be the minimum volume ellipsoid that contains $\mathcal{B}_\tau \setminus \mathcal{K}_{\tau} $ and let $\mathcal{X}_{\tau+1} = \mathcal{X}_{\tau} \cap \mathcal{B}_{\tau+1}'$. Finally, by applying an affine transformation to $\mathcal{B}_{\tau+1}'$ we obtain $\mathcal{B}_{\tau+1}$. 

Before proving that the algorithm achieves low regret we discuss the computational aspects of the algorithm. The most computationally intensive steps are cone-cutting, and the isotropic transformation that transforms $B_{\tau+1}'$ into a sphere $B_{\tau+1}$. These steps are analogous to the implementation of the ellipsoid algorithm. In particular, there is an equation for $B_{\tau+1}'$ see \cite{goldfarb1982modifications}. The affine transformations can be computed via rank one matrix updates and  therefore the computation of inverses can be done efficiently. 

\begin{algorithm*}[]
\renewcommand{\thealgorithm}{}
\caption{\textbf{2} $(X\subset \mathbb{R}^d )$}
\label{alg:algo_2_d}
\begin{algorithmic}
\STATE Input: $X$, constants $c_1$ and $c_2$, functions $\Delta_{\tau}(\gamma)$ and $\hat{\Delta}_{\tau}(\gamma)$, and total number of time-steps $T$ 
\STATE Let $\mathcal{X}_1 = X$
	\FOR{epoch $\tau=1,2,...$}
	\STATE Round $\mathcal{X}_t$ so $\mathcal{B}(r_{\tau})\subseteq \mathcal{X}_\tau 			\subseteq \mathcal{R}(R_{\tau})$, $R_{\tau}$ is minimized and $r_{\tau}:= 			R_{\tau}/(c_1 d)$. Let $\mathcal{B}_{\tau}=\mathcal{B}(R_\tau) $.
	\STATE Build a simplex with vertices $x_1,...,x_{d+1}$ on the surface of $						\mathcal{B}(r_{\tau})$.
		\FOR{round $i=1,2,...$}
			\STATE Let $\gamma_i :=2^{-i}$
			\STATE Play $x_j$ for each $j=1,...,d+1$,  $\kappa \frac{\ln(T/(\alpha 						\gamma))}{\alpha ^2\gamma_i^2 }$ times and build CI's: $							[\hat{C}_{\alpha}[F](x_j)- \gamma_i, \hat{C}_{\alpha}[F](x_j)+ 						\gamma_i]$
			\STATE Let $y_1 := \arg\max_{x_j}LB_{\gamma_i}(x_j)$
			\FOR{pyramid $k=1,2,...$}
				\STATE Construct pyramid $\Pi_k$ with apex $y_k$; let $z_1,...,z_d						$ be the vertices of the base of $\Pi_k$ and $z_0$ be the 							center of $\Pi_k$
				\LOOP
				\STATE Play each of $\{y_k,z_0,z_1,...,z_d\}$,  $\kappa \frac{\ln(T/							(\alpha \gamma))}{\alpha ^2\gamma_i^2 }$  times and 							build CI's
				\STATE Let $\textsc{center}:=z_0$, $\textsc{apex}:=y_k$, $								\textsc{top}$ be the vertex $v$ of $\Pi_k$ maximizing 							$LB_{\hat{\gamma}}(v)$, \textsc{bottom} be the vertex $v							$ of $\Pi_k$ minimizing $LB_{\hat{\gamma}}(v)$
				\IF{$LB_{\hat{\gamma}}(\textsc{top}) \geq UB_{\hat{\gamma}}(\textsc{bot}) + \Delta_{\tau}(\hat{\gamma})$ and $LB_{\hat{\gamma}}(\textsc{top}) \geq UB_{\hat{\gamma}}(\textsc{apex})+\hat{\gamma}$: (Case 1a) )}
				\STATE Let $y_{k+1}:=\textsc{top}$, immediately continue to 								pyramid $k+1$
				\ELSIF{$LB_{\hat{\gamma}}(\textsc{top}) \geq UB_{\hat{\gamma}}						(\textsc{bot}) + \Delta_{\tau}(\hat{\gamma})$ and 								$LB_{\hat{\gamma}}(\textsc{top}) < UB_{\hat{\gamma}}							(\textsc{apex})+\hat{\gamma}$: (Case 1b) )}
				\STATE Set $(\mathcal{X}_{\tau+1},\mathcal{B}_{\tau+1}) = 							\textsc{cone-cutting}(\Pi_k,\mathcal{X}_\tau,\mathcal{B}_\tau)$, 					proceed to epoch $\tau+1$
				\ELSIF{ $LB_{\hat{\gamma}}(\textsc{top}) < UB_{\hat{\gamma}}							(\textsc{bot}) + \Delta_{\tau}(\hat{\gamma})$ and 								$UB_{\hat{\gamma}}(\textsc{cent}) \geq LB_{\hat{\gamma}}						(\textsc{bot})-\bar{\Delta}_{\tau}(\hat{\gamma})$: (Case 2a) )}
					\STATE Let $\hat{\gamma}:= \hat{\gamma}/2$
					\IF{$\hat{\gamma}<\gamma_i$}
						\STATE Start next round $i+1$
					\ENDIF
				\ELSIF{ $LB_{\hat{\gamma}}(\textsc{top}) < UB_{\hat{\gamma}}							(\textsc{bot}) + \Delta_{\tau}(\hat{\gamma})$ and 								$UB_{\hat{\gamma}}(\textsc{cent}) < LB_{\hat{\gamma}}							(\textsc{bot})-\bar{\Delta}_{\tau}(\hat{\gamma})$: (Case 							2b) )}
					\STATE Set $(\mathcal{X}_{\tau+1},\mathcal{B}_{\tau+1}) = 							\textsc{hat-raising}(\Pi_k, \mathcal{X}_{\tau}, \mathcal{B}								_{\tau})$ and proceed to epoch $\tau+1$
				\ENDIF
				\ENDLOOP
			\ENDFOR
		\ENDFOR
	\ENDFOR
\end{algorithmic}
\end{algorithm*}

\begin{algorithm*}[]
\renewcommand{\thealgorithm}{}
\caption{\textsc{cone-cutting}}
\label{cone_cutting}
\begin{algorithmic}
\STATE Input: pyramid $\Pi$ with apex $y$, (rounded) feasible region $\mathcal{X}_{\tau}$ for each epoch $\tau$, enclosing ball $\mathcal{B}_{\tau}$
\STATE 1. Let $z_1,...,z_d$ be the vertices of the base of $\Pi$, and $\phi$ the angle at its apex.
\STATE 2. Define the cone $\mathcal{K}_{\tau} = \{x| \exists \lambda >0, \alpha_1,...,\alpha_d>0, \sum_{i=1}^d \alpha_i=1, x= y-\lambda\sum_{i=1}^d \alpha_{i}(z_i-y)\}$
\STATE 3. Set $\mathcal{B}_{\tau+1}'$ to be the minimum volume ellipsoid containing $\mathcal{B}_\tau  \setminus \mathcal{K}_{\tau}$ 
\STATE 4. Set $\mathcal{X}_{\tau+1}=\mathcal{X}_{\tau} \cap \mathcal{B}_{\tau+1}'$
\STATE Output: Output: new feasible region $\mathcal{X}_{\tau+1}'$ and enclosing ellipsoid $\mathcal{B}_{\tau+1}'$ 
\end{algorithmic}
\end{algorithm*}

\begin{algorithm*}[]
\renewcommand{\thealgorithm}{}
\caption{\textsc{hat-raising}}
\label{hat_raising}
\begin{algorithmic}
\STATE Input: pyramid $\Pi$ with apex $y$, (rounded) feasible region $\mathcal{X}_{\tau}$ for each epoch $\tau$, enclosing ball $\mathcal{B}_{\tau}$
\STATE 1. Let \textsc{cent} be the center of $\Pi$
\STATE 2. Set $y'=y+(y - \textsc{cent})$
\STATE 3. Set $\Pi'$ to be the pyramid with apex $y'$ and same base as $\Pi$
\STATE 4. Set $(\mathcal{X}_{\tau+1}, \mathcal{B}_{\tau+1}')=\textsc{cone-cutting}(\Pi', \mathcal{X}_{\tau}, \mathcal{B}_{\tau})$
\STATE Output: new feasible region $\mathcal{X}_{\tau+1}'$ and enclosing ellipsoid $\mathcal{B}_{\tau+1}'$ 
\end{algorithmic}
\end{algorithm*}

We follow \cite{agarwal2011stochastic} for the analysis of the algorithm. The main difference in the analysis is that we must build estimates of the $CVaR$ of the random loss at every point instead of building them for the expected loss. Because of this, we have to use different concentration results which directly affect how many times we must choose an action.

In this section we will first prove the correctness of the algorithm and then bound the regret. As in the $1$-dimensional case we work conditioned on $\mathcal{E}$ which is defined as the event that for every epoch and every round $i$, $h(x) \in [LB_{\gamma_i}(x),UB_{\gamma_i}(x)] $ for all $x$ played in that round. We will assume that
\begin{align}
\Delta_\tau (\gamma) = \big( \frac{6c_1d^4}{c_2^2}+3\big)\gamma \text{ and } \bar{\Delta}_{\tau}(\gamma) = \big( \frac{6c_1d^4}{c_2^2} + 5\big) \gamma
\end{align}
and $c_1\geq 64$, $c_2\leq 1/32$. 

\subsubsection{Correctness of the Algorithm}

In the next sequence of lemmas we show that whenever the cone-cutting procedure is carried out we do not discard all the approximate optima of $h$. We also show that the hat-raising step does what we claim. 

For the next two lemmas we assume that the distance from apex $y$ of any $\Pi$ built in epoch $\tau$ to the center of $\mathbb{B}(r_\tau)$ is at least $r_\tau/d$. That the previous is true will be shown later. 

\begin{lemma}\label{lemma5Ag}
Let $\mathcal{K}_\tau$ be the cone that will be discarded in epoch $\tau$ through case 1b) in round $i$. Let �\textsc{bottom} be the lowest CI of pyramid $\Pi$. Assume the distance from the apex $y$ to the center of $\mathbb{B}(r_\tau)$ is at least $r_{\tau}/d$. Then $h(x)\geq h(\textsc{bottom})+\gamma_i$ $\forall x \in \mathcal{K}_{\tau}$.
\end{lemma}

\begin{proof} 
Let $x$ be a point in $\mathcal{K}_{\tau}$. By construction, there exists a point $z$ in the base of the pyramid such that  $x= \alpha z + (1-\alpha)y$ for some $\alpha \in (0,1]$. Using the convexity of $h$, the fact that $z$ is in the base, and the fact that we are in case 1b), we have the two following inequalities
\begin{align*}
&h(z)\leq h(\textsc{top})\leq h(y) + 3\hat{\gamma} \\
& h(y)\geq h(\textsc{bottom}) + \Delta_{\tau}(\hat{\gamma})-2\hat{\gamma} 
\end{align*} 
 where $\hat{\gamma}$ is the CI level used for the pyramid. Since $h$ is convex we have
 \begin{align*}
 h(y)\leq \alpha h(z) + (1-\alpha)h(x)\leq \alpha (h(y)+3\hat{\gamma}) + (1-\alpha)h(x).
 \end{align*}
 Which implies 
 \begin{align*}
 h(x) \geq h(y) -3\frac{\alpha}{1-\alpha} \hat{\gamma} > h(\textsc{bottom}) + \Delta_{\tau}(\hat{\gamma}) - 3 \frac{\alpha}{1-\alpha} \hat{\gamma} -2\hat{\gamma}. 
 \end{align*}
We know $\alpha/(1-\alpha) = ||y-x||/||y-z||$. Since $x\in \mathbb{B}(R_{\tau})$, $||y-x||\leq 2 R_{\tau} = 2 c_1 d r_{\tau}.$ Moreover, $||y-z||$ is at least the height of $\Pi$, which by Lemma \ref{lemma_1_pyramids} in the Appendix, is at least $r_{\tau}c_2^2/d^3$. Thus
\begin{align*}
\frac{\alpha}{1-\alpha} = \frac{||y-x||}{||y-z||} \leq \frac{2 c_1 d r_{\tau}}{r_{\tau}c_2^2/d^3}.
\end{align*}
This implies
\begin{align*}
h(x)>h(\textsc{bottom})+\Delta_{\tau}(\hat{\gamma}) - 2\hat{\gamma} - \frac{6c_1d^4}{c_2^2}\hat{\gamma}\geq h(\textsc{bottom})+\gamma_i
\end{align*}
as required.
\end{proof}

\begin{lemma}\label{lemma6Ag}
Let $\Pi'$ be the pyramid built using the hat-raising procedure with apex $y'$ and the same base as $\Pi$ in round $i$ of epoch $\tau$. let $\mathcal{K}_{\tau}'$ be the cone to be removed. Assume the distance from $y$, the apex of $\Pi$ to the center of $\mathbb{B}(r_\tau)$ is at least $r_\tau/d$. Then $\Pi'$ has angle $\bar{\phi}$ at the apex with $\cos{\bar{\phi}}\leq 2c_2/d$, height at most $2r_\tau c_1^2/d^2$, and every point $x$ in $\mathcal{K}_\tau'$ satisfies $h(x)\geq h(x^*) + \gamma_i$.
\end{lemma}

\begin{proof}
Let $y' = y + (y-\textsc{center})$ be the apex of $\Pi'$. Let $g$ be the height of $\Pi$ (the shortest distance from the apex to the base), let $g'$ be the height of $\Pi'$ and let $b$ be the distance from any vertex in the base to the center of the base. By Lemma \ref{lemma_1_pyramids} in the Appendix we have $g' < 2g\leq 2r_\tau c_1^2/d^2$. Since $\cos{\phi} = g/\sqrt{h^2+b^2}=c_2/d$ we have $\cos{\bar{\phi}}=g'/\sqrt{g'^2+b^2}\leq 2g/\sqrt{g^2+b^2}=2\cos{\phi}=2c_2/d$.

We now show that for all $x\in \mathcal{K}'_\tau$ we have $h(x)\geq h(x^*) + \hat{\gamma}$. Since $h$ is convex we have $h(y)\leq (h(y) + h(\textsc{center}))/2$ therefore $h(y')\geq 2h(y)-h(\textsc{center})$. Since we are in case 2b) we know $h(\textsc{center})\leq h(y)-\bar{\Delta}_{\tau}(\hat{\gamma})$, so
\begin{align}
h(y')\geq h(y)+\bar{\Delta}_{\tau}(\hat{\gamma}).
\end{align}
Since we are under case 2b) we have $h(y)>h(\textsc{top}) - \Delta_{\tau}(\hat{\gamma}) - 2\hat{\gamma}>h(x)-\Delta_{\tau}(\hat{\gamma})-2\hat{\gamma}$ for all $x\in \Pi$. We therefore have that for any $z$ in the base of $\Pi$,
\begin{align}
h(y')>h(z) + \bar{\Delta}_{\tau}(\hat{\gamma}) - \Delta_{\tau}(\hat{\gamma})-2\hat{\gamma}\geq h(z),
\end{align}
where we used the settings of $\Delta_{\tau}(\hat{\gamma})$ and $\bar{\Delta}_{\tau}(\hat{\gamma})$. Finally, for any $x\in \mathcal{K}_{\tau}'$ there exists $\alpha \in [0,1)$ and $z$ in the base of $\Pi'$ such that $y' = \alpha z + (1-\alpha) x$, by convexity we have $h(y')\leq \alpha h(z) + (1-\alpha) h(x) \leq \alpha h(y') + (1-\alpha)h(x)$. The previous implies $h(x)\geq h(y')\geq h(y)+\bar{\Delta}_{\tau}(\hat{\gamma})\geq h(x^*)+\gamma_i$.
\end{proof}

\subsubsection{Regret Analysis}
As in the 1-dimensional case, to bound the total pseudo-regret $(T\bar{\mathcal{R}}_T)$ we must bound the regret incurred in a round and then bound the total number of epochs. In this section, for ease of reading we refer to quantity $T\bar{\mathcal{R}}_T$ as the regret.

\subsubsection{Bounding the regret incurred in a round.} 
We first bound the regret in round $i$ if case 2a) takes place. As before, we let $\Pi$ be a pyramid built by the algorithm with angle $\phi$, apex $y$, base $z_1,...,z_d$ and center \textsc{center}. recall that the pyramids built by the algorithm are such that the distance from the center to the base is at least $r_\tau c_2^2/d^3$. 

\begin{lemma}\label{regret_pyramid}
Suppose the algorithm reaches case 2a) in round $i$ of epoch $\tau$, assume $x^*\in \mathcal{B}(R_\tau)$, where $x^*$ minimizes $h$. Let $\Pi$ be the current pyramid and $\hat{\gamma}$ be the current width of the CI. Assume the distance from the apex of $\Pi$ to the center of $\mathcal{B}(r_\tau)$ is at least $r_\tau/d$. Then the regret incurred while playing on $\Pi$ in round $i$ is no more than
\begin{align*}
\frac{\kappa d \ln(T/(\alpha \hat{\gamma}))}{\alpha^2 \hat{\gamma}}\big(\frac{4d^7c_1}{c_2^3}+ \frac{d(d+2)}{c_2}\big) \big( \frac{12c_1d^4}{c_2^2}+11\big).
\end{align*}
\end{lemma}

\begin{proof}
The proof follows by convexity. We will first bound the variation of $h$ in the pyramid and then bound the regret on the round depending on wether $x^*$ is in $\Pi$ or not.\\
Since $\Pi$ is a convex set we know that the function value on any point in $\Pi$ is bounded above by the maximum function value at the vertices. Case 2a) implies that for any vertex its function value is bounded above by $h(\textsc{center}) + \Delta_{\tau}(\hat{\gamma}) + \bar{\Delta}_{\tau}(\hat{\gamma}) + 3\hat{\gamma}$. The previous implies that for all $x\in \Pi$ we have
\begin{align*}
h(x) \leq h(\textsc{center})+\Delta_{\tau}(\hat{\gamma}) + \hat{\Delta}_{\tau}(\hat{\gamma}) + 3\hat{\gamma}.
\end{align*}
We let $\delta := \Delta_{\tau}(\hat{\gamma}) + \hat{\Delta}_{\tau}(\hat{\gamma}) + 3\hat{\gamma}$. Let $x\in \Pi$, let $b$ be the a point in the base of $\Pi$ such that $\textsc{center} = \alpha x + (1-\alpha)b$ for some $\alpha\in[0,1]$. We know that $(1-\alpha)/\alpha = ||\textsc{center}-x||/||\textsc{center}-b||$. Since the furthest $x$ can be from  $\textsc{center}$ is when $x$ is a vertex, and the distance from \textsc{center} to $b$ is at least the radius of the largest ball inscribed in $\Pi$ with center \textsc{center}, by Lemma \ref{lemma_2_pyramids} in the  Appendix we have
\begin{align*}
\frac{1-\alpha}{\alpha} = \frac{||\textsc{center}-x||}{||\textsc{center}-b||}\leq \frac{d(d+1)}{c_2}
\end{align*}
Since $h$ is convex and we have a bound on all the function values over $\Pi$ we have
\begin{align*}
h(\textsc{center}) \leq \alpha h(x) + (1-\alpha) h(b) \leq \alpha h(x) + (1-\alpha)(h(\textsc{center})+\delta). 
\end{align*}
This implies 
\begin{equation} \label{eq7agr}
h(x) \geq h(\textsc{center}) - \frac{d(d+1)\delta}{c_2}.
\end{equation}
Combining the previous two equations we have that for any $x,x' \in \Pi$
\begin{align*}
|h(x)-h(x')|\leq \frac{d(d+2)\delta}{c_2}.
\end{align*}
Consider the case when $x^* \in \Pi$ . Since in a given round we sample $d+2$ points in the pyramid, each of them only $\kappa \ln(T/(\alpha \hat{\gamma}))/(\alpha^2 {\hat\gamma}^2))$ we have that the total regret incurred when sampling the pyramid is no more than
\begin{align*}
(d+2)(\frac{d(d+2)\delta}{c_2}) (\frac{\kappa \ln(T/(\alpha \hat{\gamma}))}{\alpha^2 \hat{\gamma}^2}).
\end{align*}
We now consider the case where $x^*\notin \Pi$. Recall that we always have $x^*\in \mathcal{B}_\tau$ by Lemma \ref{lemma5Ag}. Thus we can write $b = \alpha x^* + (1-\alpha) \textsc{center}$, for some $\alpha \in[0,1]$ where $b$ is a point in some face of the current pyramid. We know $\alpha = ||\textsc{center}-b||/||\textsc{center}-x^*||$. Using the triangle inequality we have $||\textsc{center}-x^*||\leq 2 R_\tau= 2c_1d r_\tau$. We also know that $||\textsc{center} - b||$ is at least the radius of the largest ball inscribed in $\Pi$ which by \ref{lemma_2_pyramids} in the Appendix is at least $r_\tau c_2^2 /(2d^4)$. Using the convexity of $h$ and Equation (\ref{eq7agr}) we have
\begin{align*}
h(\textsc{center})- \frac{d(d+2)\delta}{c_2}\leq h(b) \leq \alpha h(x^*) + (1-\alpha)h(\textsc{center}).
\end{align*}
Thus, $\forall x\in \Pi$ we have
\begin{align*}
h(x^*) \geq h(\textsc{center}) - \frac{d(d+1)\delta}{c_2 \alpha} \geq h(\textsc{center}) - \frac{4d^7c_1\delta}{c_2^3}\geq h(x) - \frac{4d^7c_1\delta}{c_2^3} - \frac{d(d+2)\delta}{c_2}.
\end{align*}
Using the same argument as before we know that the regret incurred in the round while evaluating points in $\Pi$ is no more than
\begin{align*}
(d+2)(\frac{4d^7c_1\delta}{c_2^3} + \frac{d(d+2)\delta}{c_2} )(\frac{\kappa \ln(T/(\alpha \hat{\gamma}))}{\alpha^2 \hat{\gamma}^2}).
\end{align*}
Plugging in $\Delta_{\tau}(\hat{\gamma})$ and $\bar{\Delta}_{\tau}(\hat{\gamma})$ yields the result. 
\end{proof}

Lemma \ref{regret_pyramid} is important because it implies that whenever we sample from a pyramid using $\hat{\gamma}$ we were in Case 2a) with $2\hat{\gamma}$ and the regret incurred is only $poly(d)/\hat{\gamma}$. The exception is when we are in the first round, however since $h$ is 1-Lipschitz the previous claim holds trivially.

We now show that we only visit Case 1a) only a bounded number of times in every round. The intuition is that every time Case 1a) occurs and we build a new pyramid its center will be closer to the center of $\mathcal{B}(R_{\tau})$ and at some point the pyramid will be inside the simplex we built at the beginning of the epoch for which we know $h$ at its vertices. 

\begin{lemma}\label{at_any_round_the_number}
At any round, the number of visits to Case 1a) is at most $2d^2\ln(d)/c_2^2$, and every pyramid build by the algorithm with apex $y$ satisfies $||y-x_0||\geq r_\tau/d$.
\end{lemma}

\begin{proof}
By definition of Case 1a) $\textsc{top}\neq y$, without loss of generality we assume $\textsc{top} = z_1$. By construction we have
\begin{align*}
||z_1-x_0|| = \sin(\phi)||y-x_0||.
\end{align*} 
Since this holds every time we enter Case 1a), we know that the total number of visits $k$ satisfies
\begin{align*}
||z_1-x_0|| = (\sin (\phi))^k r_{\tau}
\end{align*}

where $r_\tau$ is the radius of the ball where the simplex is inscribed at the beginning of round $\tau$. We also notice that for a simplex of radius $r_\tau$ the largest ball inscribed in it has radius $r_{\tau}/d$. Additionally, by construction we have $\cos(\phi)=c_2/d$ and therefore $\sin(\phi)=\sqrt{1-c_2^2/d}\leq 1-c_2^2/(2d^2)$. Therefore,   $k = 2d^2\ln(d)/c_2^2 $ ensures $||z_1-x_0||\leq r_\tau/d$ which implies that $z_1$ lies inside the simplex we build at the beginning of round $\tau$. 

Let $y_1,...,y_k$ be the apexes of the pyramids built in round $\tau$. By construction we have 
\begin{align*}
h(z_1)\geq h(\textsc{top}) \geq h(y_k)\gamma \geq h(y_{k-2})2\gamma \geq... \geq h(y_1) + k\gamma. 
\end{align*} 
On the other hand, by definition of $y_1$ we have $h(y_1)\geq h(x_i)-\gamma$ for all vertices of the simplex $x_i$. Since $z_1$ is in the simplex and $h$ is convex we have
\begin{align*}
h(y_1)\geq h(z_1)-\gamma \geq h(y_1) +(k-1)\gamma
\end{align*} 
which is a contradiction unless $k\leq 1$. Therefore, if $z_1$ is not in the simplex it must be the case that $k\leq 2 d^2 \ln(d)/c_2^2$.
\end{proof}

Using the Lemma \ref{at_any_round_the_number} we will bound the regret incurred in a round whenever it terminates in Case 2a). 

\begin{lemma}\label{for_any_round_with_CI}
For any round with CI width of $\gamma$ that terminates in Case 2a) the total regret incurred in the round is no more than
\begin{align*}
\frac{\kappa d \ln(T/(\alpha \gamma))}{\alpha^2 \gamma} \big( \frac{2d^2\ln(d)}{c_2^2} + 1\big) \big( \frac{4d^7c_1}{c_2^3} + \frac{d(d+2)}{c_2} \big) \big( \frac{12c_1d^4}{c_2^2}+11\big).
\end{align*}
\end{lemma}

\begin{proof}
By Lemma \ref{at_any_round_the_number} we have that for the given round, the total number of pyramids we have built is $k\leq 2d^2\ln(d)/c_2$. Then, by Lemma \ref{regret_pyramid} we know that for any point in the $k$-th pyramid the instantaneous regret is no more than
\begin{align*}
\delta:= \kappa \gamma d \big( \frac{4d^7c_1}{c_2^3}+\frac{d(d+2)}{c_2}\big) \big( \frac{12c_1d^4}{c_2^2} + 11\big).
\end{align*}
We now show that  the regret for any point we played during the round is at most $\delta$. Indeed, by construction $y_k$ is \textsc{top} of the $(k-1)$-th pyramid. By definition of Case 1a) we know that for any $x\in \Pi_{k-1}$ we have $f(x) \leq f(y_k) + \gamma$. Using this reasoning, we get that the function value at any vertex of any pyramid we have built during the round is also bounded by the function value at $y_k$. Additionally, as in the proof of the previous lemma, the function value at all the vertices of the simplex we built at the beginning of the epoch is also bounded by the function value at $y_k$. Since in every pyramid (and the initial simplex) we sample $d+2$ points we know that the total number of points we will play at is no more than $(d+2)(2d^2/(c_2^2 \ln(d))+1)$. To bound the total number of times we play a point we notice that for a CI with width $\hat{\gamma}$ we play it $\kappa \ln(T/(\alpha \gamma))/(\alpha^2 \hat{\gamma}^2)$. Suppose $\gamma = 2^{-i}$, since $\hat{\gamma}$ is geometrically decreased to $\gamma$ we know that the total number of plays at any point is bounded by 
\begin{align*}
 \sum_{j=1}^i \frac{\kappa \ln(T/(\alpha \gamma))}{\alpha^2 2^{-2j}} \leq \frac{4\kappa \ln(T/(\alpha \gamma))2^{2i}}{\alpha^2} = \frac{4\kappa \ln(T/(\alpha \gamma))}{\alpha^2 \gamma^2}
 \end{align*}
 Putting everything together we get that the total regret incurred during the round is no more than
 \begin{align*}
\frac{\kappa d \ln(T/(\alpha \gamma))}{\alpha^2 \gamma} \big( \frac{2d^2\ln(d)}{c_2^2} + 1\big) \big( \frac{4d^7c_1}{c_2^3} + \frac{d(d+2)}{c_2} \big) \big( \frac{12c_1d^4}{c_2^2}+11\big).
\end{align*}
\end{proof}

Using Lemma \ref{for_any_round_with_CI} we will now bound the total regret incurred at any round. 

\begin{lemma}\label{for_any_round_that_terminates}
For any round that terminates in a CI with width $\gamma$, the total regret over the round is no more than
 \begin{align*}
 \frac{\kappa d\ln(T/(\alpha \gamma))}{\alpha^2 \gamma} \big( \frac{2d^2\ln(d)}{c_2^2}+1\big) \big( \frac{4d^7c_1}{c_2^3} + \frac{d(d+2)}{c_2} \big) \big( \frac{12c_1d^4}{c_2^2} + 11\big).
 \end{align*}
\end{lemma}

\begin{proof}
We just need to bound the regret when the round ends in Case 1b) or 2b). By the definition of the algorithm, whenever a round has level $\gamma$ it must be the case that in the previous round the level was $2\gamma$ and thus using the previous lemma we can bound the regret. The exception is in the first round when $\gamma = 1$, in this case using the Lipschitz assumption we know that the instantaneous regret is no more than 1.

Because of the previous we have that the instantaneous regret at any point of the simplex we build is no more than
 \begin{align*}
 2\gamma \big( \frac{4d^7c_1}{c_2^3} + \frac{d(d+2)}{c_2} \big) \big( \frac{12c_1d^4}{c_2^2} + 11\big).
 \end{align*}
Now, if the algorithm was in Cases 1a), 1b) , or 2b) with level $\hat{\gamma}$, then it must have been in Case 2a) with level $2\hat{\gamma}$. And thus, using the bound on the regret whenever a round ends through Case 2a), we have that the instantaneous regret on the vertices any pyramid is no more than
\begin{align*}
2\hat{\gamma} \big( \frac{4d^7c_1}{c_2^3} + \frac{d(d+2)}{c_2} \big) \big( \frac{12c_1d^4}{c_2^2} + 11\big),
\end{align*}
and by using the same argument as in the proof of Lemma \ref{for_any_round_with_CI}, the number of plays at a given point is bounded above by $\kappa \ln(T/(\alpha \gamma))/(\alpha^2 \hat{\gamma}^2)$. Therefore, the total regret incurred at any pyramid built by the algorithm is no more than
\begin{align*}
 \frac{\kappa d\ln(T/(\alpha \hat{\gamma}))}{\alpha^2 \gamma} \big( \frac{4d^7c_1}{c_2^3} + \frac{d(d+2)}{c_2} \big) \big( \frac{12c_1d^4}{c_2^2} + 11\big).
 \end{align*}
Recalling the bound on the total number of pyramids built in any round yields the result.
\end{proof}

\begin{lemma}
The regret in any epoch which ends in level $\gamma$ is at most
 \begin{align*}
 \frac{\kappa d\ln(T/(\alpha \gamma))}{\alpha^2 \gamma} \big( \frac{2d^2\ln(d)}{c_2^2}+1\big) \big( \frac{4d^7c_1}{c_2^3} + \frac{d(d+2)}{c_2} \big) \big( \frac{12c_1d^4}{c_2^2} + 11\big).
 \end{align*}
 \end{lemma}
 \begin{proof}
 From Lemma \ref{for_any_round_that_terminates} we know that on any round with level $\gamma$, the regret is bounded by $C/\gamma$ where $C$ is some constant. Since $\gamma$ is reduced geometrically, the net regret on an epoch where the largest level we encounter is $\gamma$ is bounded by 
 \begin{align*}
 \sum_{j=1}^i \frac{C}{2^{-j}}\leq 2C2^i = \frac{2C}{\gamma},
\end{align*} 
  which yields the result.
 \end{proof}

\subsubsection{Bounding the Number of Epochs}

To bound the number of epochs we must show that every time \textsc{cone-cutting} is performed we discard a sufficiently large portion of the current ball. More specifically, we need to analyze the ratios of volumes of $\mathcal{B}_{\tau+1}$ and $\mathcal{B}_{\tau}$.

\begin{lemma}\label{lemma_vol_reduction}
Let $\mathcal{B}_{\tau}$ be the smallest ball containing $\mathcal{X}_{\tau}$, let $\mathcal{B}'_{\tau+1}$ be the minimum volume ellipsoid containing $\mathcal{B\tau}\setminus \mathcal{K}_{\tau}$. Then, for small enough constants $c_1,c_2$, $vol(\mathcal{B}'_{\tau+1})\leq \rho \cdot vol(\mathcal{B}_\tau)$ where $\rho = \exp(-\frac{1}{4(d+1)})$. 
\end{lemma}

\begin{proof}
This result is analogous to the volume reduction results for the ellipsoid method with a gradient oracle. It is easy to see that it suffices to consider the intersection of $\mathcal{B}_{\tau}$ with a half-space in order to understand the set $\mathcal{B}_{\tau} \setminus \mathcal{K}_{\tau}$. This is because if we were to discard only the spherical cap instead of the whole cone then the minimum enclosing ellipsoid would increase its volume.

The previous choices of $c_1,c_2$ guarantee that  the distance from the center of $\mathcal{B}_{\tau}$ to the origin is at most $R_{\tau}/(4(d+1))$. The previous is true because by construction the apex of cone $\mathcal{K}_\tau$ is always contained in $\mathbb{B}(r_\tau)$, and the height of the cone is at most $R_\tau \cos (\bar{\phi})\leq R_\tau / (8(d+1))$ again by construction. Thus, if $r_\tau \leq R_\tau/(32(d+1))$, then the distance of the hyperplane to the origin is at most $R_{\tau}/(4(d+1))$.\\
Therefore, $\mathcal{B}_{\tau+1}'$ is the minimum volume ellipsoid that contains the intersection of $\mathcal{B}_{\tau}$ with a hyperplane that is at most $R_{\tau}/(4(d+1))$ from its center. Using Theorem 2.1 from \cite{goldfarb1982modifications} (with $\alpha=-1/(4(d+1))$) we get the result.
\end{proof}
\begin{lemma}
At any epoch with CI level $\gamma$, the instantaneous regret of any point in $\mathcal{K}_{\tau}$ is at least $\gamma$.
\end{lemma}

\begin{proof}
Since every epoch terminates only through Cases 1b) or 2b) we only check the claim is true for these two cases. If the epoch ends through Case 1b) the proof of Lemma \ref{lemma5Ag} gives the result. If the epoch ends through Case 2b), after \textsc{hat-raising} we now that the apex $y'$ of pyramid $\Pi'$ satisfies $h(y')\geq h(z_i) + \gamma$ for all vertices $z_1,...,z_d$ of the pyramid. Writing $y' = \alpha x + (1-\alpha)z$ with $x$ in $\mathcal{K}_\tau$, $z$ in the base of $\Pi'$ and $\alpha \in [0,1]$, we can conclude that $h(x)\geq h(x^*) + \gamma$ just as we did in the proof of Lemma \ref{lemma6Ag}.
\end{proof}

We are now ready to bound the total number of epochs. 

\begin{lemma}
The total number of epochs in the algorithm is no more than $\frac{d\ln(T)}{\ln(1/\theta)}$ where $\theta = \exp(-\frac{1}{4(d+1)})$.
\end{lemma}

\begin{proof}
Recall $x^*$ is the minimizer of $h$. Since $h$ is 1-Lipschitz, any point inside a ball or radius $1/\sqrt{T}$  centered around $x^*$ has instantaneous regret of at most $1/\sqrt{T}$. The volume of this ball is $T^{-d/2}V_d$, with $V_d$ equal to the volume of the unit ball in $d$-dimensions. Suppose the algorithm goes through $k$ epochs. By Lemma \ref{lemma_vol_reduction} we know that the volume of $\mathcal{X}_\tau$ is bounded above by $\rho^k V_d$. By the previous lemma we know that the instantaneous regret of any point that was discarded had instantaneous regret at least $1/\sqrt{T}$. This is because at any given epoch and round we sample at $\frac{\kappa \ln(T/(\alpha \gamma))}{\alpha^2 \gamma^2}$ and this quantity can not be more than $T$. Because of the previous, any point in the ball centered at $x^*$ with radius $1/\sqrt{T}$ is never discarded. Therefore the algorithms stops whenever 
\begin{align*}
\theta^k V_d \leq T^{-d/2} V_d
\end{align*} 
implying $k \leq \frac{d\ln(T)}{\ln(1/\theta)}$.
\end{proof}

We are now ready to prove Theorems \ref{thm:pseudo_regret_algo2_full_dim} and \ref{thm:regret_algo2_full_dim}.

\begin{proof}[Proof of Theorem \ref{thm:pseudo_regret_algo2_full_dim}]
Using the bound on the regret incurred in an epoch and the fact that $\gamma\geq 1/\sqrt{T}$ we know the total regret on an epoch is no more than
 \begin{align*}
 \frac{\kappa d \sqrt{T} \ln(T/\alpha)}{\alpha^2} \big( \frac{2d^2\ln(d)}{c_2^2}+1\big) \big( \frac{4d^7c_1}{c_2^3} + \frac{d(d+2)}{c_2} \big) \big( \frac{12c_1d^4}{c_2^2} + 11\big).
 \end{align*}
 By the previous lemma we know the total number of epochs is no more than $d\ln(T)/\ln(1/\theta)$. Thus the total regret $T\bar{\mathcal{R}}_T$ is bounded above by
 \begin{align*}
  \frac{\kappa d^2 \sqrt{T} \ln(T/\alpha) \ln(T)}{\alpha^2 \ln(1/\theta)} \big( \frac{2d^2\ln(d)}{c_2^2}+1\big) \big( \frac{4d^7c_1}{c_2^3} + \frac{d(d+2)}{c_2} \big) \big( \frac{12c_1d^4}{c_2^2} + 11\big).
 \end{align*} 
 Recall that we were working conditioned on $\mathcal{E}$. As in the proof of the 1-dimensional algorithm, we have $P(\mathcal{E}')\leq 1/T$. Plugging in the value of $\theta$ above yields the result. 
\end{proof}

\begin{proof}[Proof of Theorem \ref{thm:regret_algo2_full_dim}] 
The proof is almost the same as the one of Theorem \ref{thm:second_theorem_algo2_1d} with two slight differences. First, we use Theorem \ref{thm:pseudo_regret_algo2_full_dim}, instead of \ref{thm:first_theorem_algo2_1d} to bound $\bar{\mathcal{R}}_T$. Second, using the same argument as in the proof of Theorem  \ref{thm:second_theorem_algo2_1d} we get that with probability at least $1-\frac{2}{T},$ $CE = \tilde{O} (\frac{\sqrt{d}}{\alpha \sqrt{T}})$. 
\end{proof}

\subsection{Analysis of Algorithm 3}

The following algorithm, a generalization of Algorithm 1, will guarantee vanishing $\bar{\mathcal{R}}_T^\rho$ and $\mathcal{R}_T^\rho$ by exploiting the Kusuoka representation of risk measure $\rho$. 

\begin{algorithm*}[]
\renewcommand{\thealgorithm}{}
\caption{\textbf{3}}
\label{alg:algo_3}
\begin{algorithmic}
	\STATE Input: $X \subset \mathbb{R}^d$, $x_1 \in X$, $z_1 \in Z$ step size $\eta$, 		$\delta$
	\FOR{$t=1,...,T$}
		\STATE Sample $u \sim \mathbb{S}^{d+N}$
		\STATE Let $u^1_t = [u_1;...;u_d] $ and $u^2_t =[u_{d+1};...;u_{d+N}]$
		\STATE Play $\tilde{x}_t:=x_t + \delta u^1$, observe $f_t(\tilde{x}_t)$
		\STATE Let $ \tilde{z}_{t} = z_{t} + \delta u^2$
		\STATE Let $g^1_t := \frac{(d+N)}{\delta} (\mathcal{G}_t(\tilde{x}_t,\tilde{z}_t)) 				u^1_t $
		\STATE Let $g^2_{t} := \frac{(d+N)}{\delta} ( \mathcal{G}_t(\tilde{x}_t,\tilde{z}				_t) ) u^2_t $
		\STATE Update $x_{t+1} \leftarrow \Pi_{X_\delta} (x_t - \eta g^1_t)$
		\STATE Update $z_{t+1} \leftarrow \Pi_{Z_\delta} (z_t - \eta g^2_t)$
	\ENDFOR
\end{algorithmic}
\end{algorithm*}
Notice that due to Lemma \ref{smooth_f}, $g_t := [g^1_t; g^2_t]$ is a one point gradient estimator of the smoothened version of $\mathcal{G}$, $\hat{\mathcal{G}}$.

The proofs of Theorems \ref{thm:first_thm_Kusuoka} and \ref{thm:second_thm_Kusuoka} will be similar to that of Theorems \ref{thm:pseudo_algo1} and \ref{thm:regret_algo1}, however we must be careful to make sure we do not introduce unnecessary factors of $N$, $d$ and $\frac{1}{\alpha}$. 

\begin{lemma}\label{bound_gradient_mathcal_G}
$||\nabla \mathcal{G}||\leq N (G+1) + 1$
\end{lemma}

\begin{proof}
\begin{align*}
|| \nabla \mathcal{G} || &= \sqrt{ \sum_{i=1}^d (\sum_{n=1}^N \mu_n \nabla_{x_i}\mathcal{L}_n  )^2 + \sum_{n=1}^N (\mu_n \nabla_{z_n}\mathcal{L}_n)^2 }\\
&\leq \sqrt{\sum_{i=1}^d ( ||\mu||_1 || \nabla_{x_i}\mathcal{L}_n ||_\infty )^2  + \sum_{n=1}^N (\mu_n \nabla_{z_n} \mathcal{L}_n )^2} \quad \text{$||. ||_\infty$ is over n=1,...,N}\\
&\leq \sqrt{ \sum_{i=1}^d ||\nabla_{x_i}\mathcal{L}_n||_\infty ^2 +\sum_{n=1}^N \mu_n \nabla_{z_n}\mathcal{L}_n^2 } \quad \text{since $\sum_{n=1}^N \mu_n = 1$, and $\mu_i \leq 1$}\\
&\leq \sqrt{\sum_{i=1}^d ||\nabla_{x_i}\mathcal{L}_n||_{\infty}^2 + \sum_{n=1}^N \mu_{n} (1+N)^2  }\\
&\leq \sqrt{\sum_{i=1}^d||\nabla_{x_i} \mathcal{L}_n||_{\infty}^2 } + \sqrt{\sum_{n=1}^N \mu_n (1+N)^2}\\
&\leq \sqrt{\sum_{i=1}^d|| N \nabla_{x_i} f ||_{\infty}^2 } + \sqrt{\sum_{n=1}^N \mu_n (1+N)^2}\\
&\leq N G + (1+N) \quad \text{since $\sum_{n=1}^N \mu_n = 1$ }
\end{align*}

\end{proof}

\begin{lemma}\label{ogd_on_mathcal_G}
Running online gradient descent on $\{\mathcal{G}_t\}_{t=1}^T$ ensures that for all $x \in X$ and all $z \in Z$
\begin{align*}
2 [\sum_{t=1}^T \mathcal{G}_t (x_t,z_t) - \sum_{t=1}^T \mathcal{G}_t(x,z)] &\leq \frac{||x_T - x^*||^2+\sum_{n=1}^d \mu_n ||z_{t,n}-z^*_n||^2}{\eta} +\\
& \eta [ \sum_{t=1}^T( ||\nabla_x \mathcal{G}_t (x_t,y_t) + \sum_{n=1}^N \mu_n |\nabla_{z_n}\mathcal{L}(x_t,z_t)|^2 ) ].
\end{align*}
\end{lemma}

\begin{proof}
\begin{align*}
& 2 [\sum_{t=1}^T \mathcal{G}_t(x_t,z_t) - \sum_{t=1}^T \mathcal{G}_t(x,z)] \\
&\leq 2 \sum_{t=1}^T \nabla \mathcal{G}_t(x_t,z_t)^\top ( [x_t;z_t] - [x;z] ) \\
& = 2 \sum_{t=1}^T \nabla_x \mathcal{G}_t(x_t,z_t)^\top (x_t-x) + 2\sum_{t=1}^T \sum_{n=1}^d \mu_n \nabla_z \mathcal{L}(x_t,z_t) (z_{t,n}.z_n)\\
&\leq \frac{||x_T - x||^2}{\eta} + \sum_{n=1}^N \mu_n \frac{||z_{T,n}-z_n||^2}{\eta} + \eta [ \sum_{t=1}^T (||\nabla_x \mathcal{G}_t|| + \sum_{n=1}^d \mu_n (\nabla_z \mathcal{L}_{t,n})^2) ] \quad \text{by Equations \ref{ogd_telescoping1} and \ref{ogd_telescoping2}}
\end{align*}
\end{proof}

\begin{lemma}\label{bandit_OGD_on_G}
Let $y^*=(x^*,z^*) \in \arg\min \mathbb{E}_{\xi}[\sum_{t=1}^T \mathcal{G}_t (x,z)]$. With appropriate  choice of parameters $\eta, \delta$ we have 
\begin{align*}
\mathbb{E}_{int} [ \sum_{t=1}^T \mathcal{G}_t(\tilde{y}_t) ] -\sum_{t=1}^T \mathcal{G}_t(y^*) \leq O(d N^{3/2} T^{3/4} )
\end{align*}
\end{lemma}

\begin{proof}
First we need a bound on $\sum_{t=1}^T \mathcal{G}_t(y^*_\delta) -\sum_{t=1}^T \mathcal{G}_t(y^*) $, where $y^*_\delta = \Pi_{X_\delta \times Z_\delta}(y^*)$. If $\mathcal{G}$ is Lipschitz $L$ with respect to some norm $||\cdot||$, by Lemma \ref{shalev_lipschitz} we have $||\nabla \mathcal{G}||_* \leq L$. For any $y = [x;z]$ with $x\in X$ and $z\in Z$, let us use $||y|| = ||x||_2 + ||z||_{\infty}$ with dual norm $||y||_* = \max\{||x||_2, ||z||_1\}$ (see Lemma \ref{lemma_dual_norm} in the Appendix). 
\begin{align*}
\sum_{t=1}^T \mathcal{G}_t(y^*_\delta) -\sum_{t=1}^T \mathcal{G}_t(y^*) &\leq T L ||y^*-y^*_\delta|| \\
&\leq \delta T L D_{\mathcal{G}}^{||\cdot||}\quad \text{by Lemma \ref{lemma_projection} in the  Appendix}\\ 
&\leq O(\delta T G N).
\end{align*}
The last inequality holds because of the following two facts, 1) $||\nabla \mathcal{G}||_* = \max\{||\nabla_x \mathcal{G}||_2, ||\nabla_z\mathcal{G}||_1\} \leq \max\{G, \sum_{n=1}^N \mu [1 + N ]\}\leq G + 1 +N $ and 2) $||y_1 - y_2|| = ||x_1-x_2||_2 + ||z_1-z_2||_\infty \leq D_X + 2 := D_{\mathcal{G}}^{||\cdot||}$. Let $\mathbb{E}_{int}$ be the expectation taken with respect to the internal randomization of the algorithm. Following the proof of Lemma \ref{algo1_L} we have

\begin{align*}
&\mathbb{E}_{int} [ \sum_{t=1}^T \mathcal{G}_t(\tilde{y}_t) ] - \sum_{t=1}^T \mathcal{G}_t(y^*) \\
& \leq \mathbb{E}_{int} [ \sum_{t=1}^T \mathcal{G}_t(y_t) ] - \sum_{t=1}^T \mathcal{G}_t(y^*) + \delta  G_{\mathcal{G}}T \quad \text{$\mathcal{G}$ is $G_\mathcal{G}$-Lipschitz and $||y-\tilde{y}||\leq \delta$}\\
& \leq \mathbb{E}_{int} [ \sum_{t=1}^T \mathcal{G}_t(y_t) ] - \sum_{t=1}^T \mathcal{G}_t(y^*_\delta) + \delta  G_{\mathcal{G}}T +O(\delta TG N)\\
& \leq \mathbb{E}_{int} [ \sum_{t=1}^T \hat{\mathcal{G}}_t(y_t) ] - \sum_{t=1}^T \hat{\mathcal{G}}_t(y^*_\delta) + 3\delta G_{\mathcal{G}}T + \delta D_{\mathcal{G}}G_{\mathcal{G}}T\quad |\mathcal{G}(y)-\hat{\mathcal{G}}(y)|<\delta G_{\mathcal{G}}\\
&\leq \frac{||x_T - x^*||_2^2}{2\eta} + \sum_{n=1}^N \mu_n \frac{||z_{t,n}-z^*_n||_2^2}{2\eta} + \mathbb{E}_{int}[2\eta [ \sum_{t=1}^T (||g^1_t||_2 + \sum_{n=1}^d \mu_n (g^2_{t,n})^2) ]] +  3\delta G_{\mathcal{G}}T + O(\delta T G N)\\
& \quad \text{reduction to bandit feedback and Lemma \ref{ogd_on_mathcal_G}}\\
&\leq \frac{D_X^2 +2}{2\eta} + 2\eta \mathbb{E}_{int}[ \sum_{t=1}^T (||g^1_t||_2^2 + \sum_{n=1}^d \mu_n (g^2_{t,n})^2) ] +  3\delta G_{\mathcal{G}}T + O(\delta T G N)\\
&\leq  \frac{D_X^2 +2}{2\eta} + 2 \eta \frac{(d+N)^2 N^2}{\delta^2} T +  3\delta G_{\mathcal{G}}T + O(\delta T G N)\\
&\leq O(dN^{3/2} T^{3/4} )
\end{align*}
where we chose $\eta = O(\frac{1}{dN^{3/2}T^{3/4}})$ and $\delta = O(\frac{N^{1/2}}{T^{1/4}})$ and plugged in the bound on $G_{\mathcal{G}}$ from Lemma \ref{bound_gradient_mathcal_G}.
\end{proof}

\begin{proof}[Proof of Theorem \ref{thm:first_thm_Kusuoka}]
Take $\mathbb{E}_{\xi}[\cdot]$ on both sides of the result in Lemma \ref{bandit_OGD_on_G} and interchange the expectations (this can be done using Fubini's Theorem and the uniform bound on $\mathcal{G}_t$). Noting that for all $x\in X$ and all $z\in [0,1]$ (in particular for every $(\tilde{x}_t,\tilde{z}_t)$) we have  
\begin{align*}
 \mathbb{E}_{\xi\sim P}[\mathcal{L}_n^t(x,z)] = z + \frac{1}{n/N}\mathbb{E}_{\xi \sim P}[f(x,\xi) - z]_+ \geq CVaR_{n/N}[F](x),
\end{align*}
it follows that since $\mathcal{G}_t(x,z) := \sum_{n=1}^N \mu_n \mathcal{L}_n^t(x,z)$ we have $\mathbb{E}_{\xi\sim P}[\mathcal{G}_t(x,z)] \geq \rho[F](x)$.  Noting that $\mathbb{E}_{\xi}[\sum_{t=1}^T \mathcal{G}_t (y^*)] = \min_{x\in X}\rho[F](x)$ we get the desired result.
\end{proof}

\begin{proof}[Proof of Theorem \ref{thm:second_thm_Kusuoka}] 
We notice that strong convexity of $f(\cdot,\xi)$ implies strong convexity of $\rho[F](\xi)$ since each of the $C_{\alpha_i}[F](\cdot)$ in the Kusuoka representation of $\rho[F]$ is strongly convex. Let $x^* =argmin_{x\in X}\rho[F](x)$. We follow the proof of Theorem \ref{thm:regret_algo1}. Let the concentration error $CE = \rho[\{f_t(x^*)\}_{t=1}^T] - \min_{x\in X}\rho[\{f_t(x)\}_{t=1}^T]$.
\begin{align*}
& \mathbb{E}[\rho[\{f_t(x_t)\}] - \min_{x\in X} \rho[\{f_t(x)\}]]\\
& =  \mathbb{E}[\rho[\{f_t(x_t)\}] \pm \rho[\{f_t(x^*)\}] - \min_{x\in X} \rho[\{f_t(x)\}]]\\
& = \mathbb{E}[ \sum_{n=1}^N \mu_n C_{n/N}[\{f_t(x_t)\}] -  \rho[\{f_t(x^*)\}]] + \mathbb{E}[CE ]\\
& \leq \mathbb{E}[ \frac{N}{T} \sum_{t=1}^T |f_t(x_t) - f_t(x^*)|]+ \mathbb{E}[CE] \quad \text{as in the last line of the proof of Theorem \ref{thm:regret_algo1} }\\ 
& \leq \frac{N}{T} \sum_{t=1}^T \mathbb{E}_t[ ||x_t-x^*|| ] +\mathbb{E} [CE]\\
& \leq \frac{N}{T}  \sqrt{T} \sqrt{ \sum_{t=1}^T\mathbb{E}_t[ ||x_t-x^*||^2 ]} +\mathbb{E}[CE]\\
& \leq \frac{N}{T} \sqrt{T} \sqrt{\sum_{t=1}^T\frac{2}{\beta} \mathbb{E}[\rho[F](x_t)-\rho[F](x^*)]} +\mathbb{E}[CE]\\
&\leq O( \frac{d^{1/2}N^{7/4}}{\beta^{1/2}T^{1/8}} )+\mathbb{E}[CE] \qed
\end{align*}
The expectation of the concentration error can be bounded as in the proof of Theorem \ref{thm:regret_algo1} by $\tilde{O}(\frac{N^{3/2}\sqrt{d}}{\sqrt{T}})$. This yields the result.
\end{proof}

\subsection{Analysis of Algorithm 4}
Recall Algorithm 4 is the modification of Algorithm 2 where we sample $\tilde{O}(\frac{N^2 \ln(NT)}{\gamma^2})$ times a point (instead of $O(\frac{\ln(T/(\alpha \gamma))}{\alpha^2 \gamma^2})$) to build a $\gamma$-CI. In this section we present the proofs of Theorems \ref{thm:third_thm_Kusuoka} and \ref{thm:fourth_thm_Kusuoka}. We only need to show that $\tilde{O}(\frac{N^2 \ln(NT)}{\gamma^2})$ samples are sufficient to build a $\gamma$-CI that holds with high probability. Afterwards it is easy to verify that the proofs of Theorems \ref{thm:pseudo_regret_algo2_full_dim} and \ref{thm:regret_algo2_full_dim} go through. 

\begin{lemma}
To build a $\gamma$-CI for $\rho[F](x)$ that holds with probability at least $1-\frac{1}{T^2}$ we need no more than $O (\frac{N\ln{(N)}\ln{(\sqrt{N}T)}}{\gamma^2})$ samples. 
\end{lemma}

\begin{proof}
Notice that
\begin{align*}
|\rho[X] - \hat{\rho}[X]| = |\sum_{n=1}^N \mu_n ( C_{n/N}[X] - \hat{C}_{n/N}[X]) | \leq \sum_{n=1}^N \mu_n |C_{n/N}[X] - \hat{C}_{n/N}| %\leq \sum_{n=1}^N \mu_n \gamma = \gamma 
\end{align*}
Therefore, if we obtain $\gamma$-CI's for each term $|C_{n/N}[X] - \hat{C}_{n/N}|$ that hold with probability at least $1-\frac{1}{NT^2}$ a union bound yields the result.
From Theorem \ref{concentration_cvar} we know that $O(\frac{N^2 \ln(\sqrt{N}T)}{n^2 \gamma^2})$ samples suffice to build a $\gamma$-CI for $C_{n/N}[X]$ that holds probability at least $1-\frac{1}{NT^2}$. Summing up the number of samples, approximating the sum with an integral and using a union bound yields the result.
\end{proof}
We are now ready to prove the theorems. 

\begin{proof}[Proof of Theorem \ref{thm:third_thm_Kusuoka}]
It is easy to see that the proof of Theorem \ref{thm:pseudo_regret_algo2_full_dim} goes through if we set $h(\cdot) = \rho[F](\cdot)$ and we replace everywhere the number of times we sample a point $O(\frac{\ln(T/(\alpha \gamma))}{\alpha^2 \gamma^2})$ with $\tilde{O}(\frac{N^2 \ln(T)}{\gamma^2})$. 
\end{proof}

\begin{proof}[Proof of Theorem \ref{thm:fourth_thm_Kusuoka}]
The proof follows from almost the same reasoning as in the proof of Theorem \ref{thm:regret_algo2_full_dim}. We have 
\begin{align*}
& \rho[\{f_t(x_t)\}_{t=1}^T] - \min_{x \in X} \rho[\{f_t(x)\}_{t=1}^T] \\
& \leq \frac{N}{ T}  \sqrt{T} \sqrt{\frac{2}{\beta}  \sum_{t=1}^T C_\alpha[F](x_t) -C_\alpha[F](x^*) } + CE \\
& \leq O(\frac{d^8 N^3 }{\beta^{1/2} T^{1/4}} ) + CE \quad \text{(with probability at least $1-\frac{1}{T})$}
\end{align*}
where $CE = \rho[F](x^*)- \min_{x\in X} \rho[\{f_t(x)\}]$ and $x^* = argmin_{x\in X} \rho[F](x)$. Just as in the proof of Theorem \ref{thm:regret_algo1} we can bound $CE$ with probability at least $1-\frac{2}{T}$ by $\tilde{O}(\frac{N^3/2 \sqrt{d}}{\sqrt{T}})$. A union bound yields the result.
\end{proof}

\end{appendices}

\end{document}